\documentclass[preprint]{article}

\PassOptionsToPackage{round}{natbib}

\usepackage{neurips_2026}

\usepackage[utf8]{inputenc}
\usepackage[T1]{fontenc}
\usepackage[protrusion=true,expansion=true]{microtype}
\usepackage{hyperref}
\usepackage{amsmath}
\usepackage{amsfonts}
\usepackage[dvipsnames,svgnames,x11names]{xcolor}
\usepackage{booktabs}
\usepackage{graphicx}
\usepackage{tikz}
\usepackage{tabularx}
\usepackage{pifont}
\usepackage{silence}

\newcommand{\cmark}{\ding{51}}
\newcommand{\xmark}{\ding{55}}

\newcolumntype{Y}{>{\centering\arraybackslash}X}

\usepackage{amsthm}
\newtheorem{theorem}{Theorem}
\newtheorem{proposition}{Proposition}
\newtheorem{corollary}{Corollary}

\usetikzlibrary{arrows.meta,positioning,fit,calc,backgrounds}

\hypersetup{colorlinks=true,citecolor=blue,urlcolor=blue}

\title{Arrow: A Foundation Model for Causal Discovery}

\author{
Ryan Thompson\thanks{These authors contributed equally.} \\
University of Technology Sydney \\
\texttt{ryan.thompson-1@uts.edu.au} \\
\And
He Zhao\footnotemark[1] \\
CSIRO's Data61 \\
\texttt{he.zhao@csiro.au}
\AND
Daniel M. Steinberg \\
CSIRO's Data61 \\
\texttt{dan.steinberg@csiro.au}
\And
Edwin V. Bonilla \\
CSIRO's Data61 \\
\texttt{edwin.bonilla@csiro.au}
}

\begin{document}

\maketitle

\begin{abstract}

We introduce \texttt{Arrow}, a foundation model for zero-shot causal discovery on observational tabular data. \texttt{Arrow} factorizes a directed acyclic graph into an undirected skeleton and a topological order, guaranteeing acyclicity by construction. Given a new dataset, it uses a transformer-based architecture to contextualize variables within and across observations, then predicts skeleton edge probabilities and node order scores that together define a graph. \texttt{Arrow} is trained in a supervised fashion on synthetic datasets with ground-truth graphs, using an end-to-end differentiable directed edge composite likelihood induced by the skeleton--order factorization. The training distribution spans diverse graph families, functional forms, noise models, and dataset shapes. Across in- and out-of-distribution synthetic, semi-synthetic, and real datasets, \texttt{Arrow} matches or outperforms existing causal discovery methods at substantially lower inference cost than competitive alternatives. Our results demonstrate that large-scale pretraining on diverse synthetic data can yield zero-shot causal discovery models that are fast, accurate, and reusable on new datasets.

\end{abstract}

\section{Introduction}
\label{sec:introduction}

Causality lies at the heart of scientific inquiry, yet underlying causal structure is rarely observed directly. Recovering such structure from data is the goal of causal discovery, typically formulated through directed acyclic graphs (DAGs) that represent direct causal relations and the conditional independencies they entail \citep{Spirtes2001,Pearl2009}. Because causal structure is often unavailable in precisely the settings where it is most useful, causal discovery has become a central problem in machine learning, with applications spanning psychology \citep{Foster2010}, economics \citep{Imbens2020}, and epidemiology \citep{Tennant2021}, as well as numerous other domains.

Historically, progress in causal discovery has been driven by task-specific methods, which solve each dataset from scratch. These methods include conditional independence testing \citep{Spirtes2001,Kalisch2007}, graph or topological order search \citep{Chickering2002,Andrews2023,Duong2025}, and continuous optimization \citep{Zheng2018,Bello2022,Nazaret2024}. Although this line of work has produced many of the field's strongest methods, task-specific methods are compute intensive and rely on assumptions tailored to the dataset at hand.

Pretrained models have recently emerged as a promising alternative to task-specific methods \citep{Lorch2022,Ke2023,Dhir2025,Wu2025,Yin2025,Peng2026}. Rather than solving causal discovery anew for every dataset, they learn generalizable inference procedures from large sets of synthetic tasks, echoing recent tabular foundation models wherein synthetic pretraining has been shown to generalize to real data \citep{Hollmann2025,Qu2025}. For such an approach to yield a credible causal discovery foundation model, however, we believe three properties are necessary: (i) zero-shot inference; (ii) guaranteed DAG outputs; and (iii) broad pretraining diversity. As shown in Table~\ref{tab:capabilities}, no existing pretrained model satisfies all three.

Motivated by the prospect of fast, accurate, and reusable causal discovery models, we introduce \texttt{Arrow}, a foundation model for zero-shot causal discovery on observational tabular data. \texttt{Arrow} acts as a direct map from datasets to DAGs, producing a graph prediction in a single forward pass. Its design combines a skeleton--order factorization of DAGs that guarantees acyclic graph predictions, a transformer architecture that contextualizes variables within and across observations, and large-scale supervised pretraining on more than 100M unique synthetic causal discovery tasks with known ground-truth graphs, spanning datasets with up to $n=2000$ observations and $p=100$ variables.

Taken together, these design choices instantiate the foundation model perspective above: \texttt{Arrow} is trained once, applied zero-shot to new datasets, structured to output DAGs, and exposed during pretraining to a broad range of causal discovery tasks. The rest of the paper develops the necessary components, and shows that they yield a model that is fast, accurate, and practically reusable. Figure~\ref{fig:synthetic_accuracy_runtime} previews this central point: on out-of-distribution tasks, \texttt{Arrow} outperforms both pretrained and task-specific baselines, occupying the most favorable point on the accuracy--runtime frontier across multiple metrics. To make \texttt{Arrow} accessible to the community, we will soon publicly release its weights, so that researchers and practitioners can use the model without the cost of pretraining.

\begin{figure}[t]
\centering
\includegraphics[width=\linewidth]{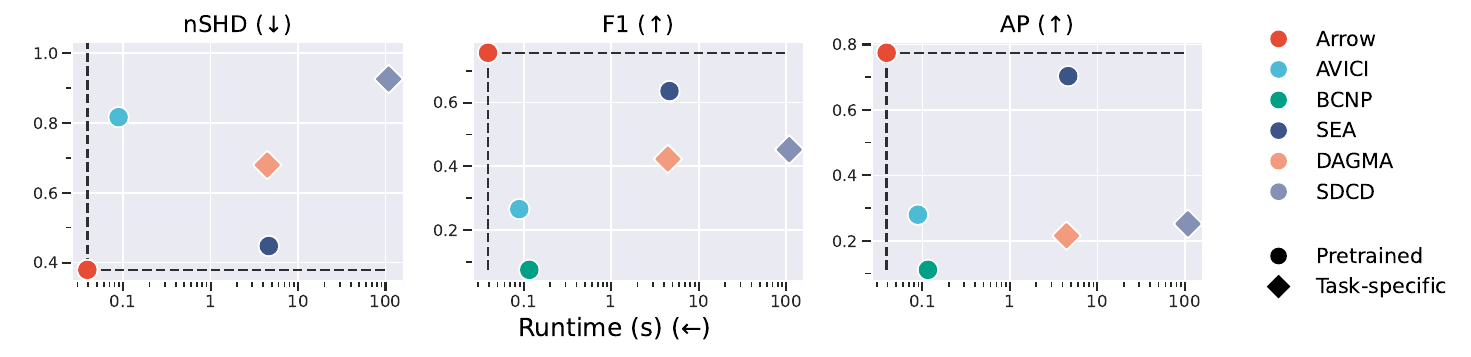}
\caption{Accuracy and runtime on out-of-distribution synthetic causal discovery tasks generated from small-world graphs, spline functions, and gamma noise, with $n\in\{100,\dots,2000\}$ observations and $p\in\{2,\dots,100\}$ variables. Dashed lines trace the Pareto frontier, marking non-dominated methods.}
\label{fig:synthetic_accuracy_runtime}
\end{figure}

\begin{table}[ht]
\centering
\small
\begingroup
\setlength{\tabcolsep}{0.1pt}
\caption{Comparison of pretrained causal discovery models. Zero-shot inference means graph prediction on new data without task-specific search, retraining, or optimization. DAG guarantee means acyclicity is enforced at inference time. Pretraining diversity means pretraining across graph families, functional forms, noise models, and dataset shapes to encourage generalization to unseen tasks. 
See Appendix~\ref{app:capabilities} for details.}
\label{tab:capabilities}
\begin{tabularx}{\linewidth}{@{}lYYYYYYY@{}}
\toprule
&
\shortstack{\texttt{AVICI}\\{\tiny\citep{Lorch2022}}}
& \shortstack{\texttt{CSIvA}\\{\tiny\citep{Ke2023}}}
& \shortstack{\texttt{BCNP}\\{\tiny\citep{Dhir2025}}}
& \shortstack{\texttt{SEA}\\{\tiny\citep{Wu2025}}}
& \shortstack{\texttt{ADAG}\\{\tiny\citep{Yin2025}}}
& \shortstack{\texttt{CauScale}\\{\tiny\citep{Peng2026}}}
& \shortstack{\texttt{Arrow}\\{\tiny(ours)}} \\
\midrule
Zero-shot inference
& \cmark & \cmark & \cmark & \xmark & \cmark & \cmark & \cmark \\
DAG guarantee
& \xmark & \xmark & \cmark & \xmark & \xmark & \xmark & \cmark \\
Pretraining diversity
& High & Low & Low & High & Low & Moderate & Very high \\
\bottomrule
\end{tabularx}
\endgroup
\end{table}

The rest of the paper is organized as follows. Section~\ref{sec:factored} defines the skeleton--order DAG factorization, derives exact and composite likelihoods, and gives a maximum-probability prediction rule. Section~\ref{sec:model} specifies the model architecture for mapping a tabular dataset to skeleton probabilities and order scores. Section~\ref{sec:pretraining} describes the synthetic task distribution and training scheme used to learn this map. Section~\ref{sec:experiments} then turns to evaluation, reporting results on in-distribution synthetic data and out-of-distribution synthetic, semi-synthetic, and real data. Finally, Section~\ref{sec:concluding} closes the paper.

\section{Likelihood formulation}
\label{sec:factored}

\subsection{Skeleton--order factorization}

\texttt{Arrow} represents a DAG through two simpler objects: an undirected skeleton described by an adjacency matrix $A\in\{0,1\}^{p\times p}$, and a topological order described by a permutation $\pi$ of $\{1,\dots,p\}$. Since the skeleton is undirected, the matrix $A$ is symmetric, with $A_{jk}=A_{kj}$ and $A_{jj}=0$, so the entry $A_{jk}$ indicates whether variables $j$ and $k$ are adjacent in the skeleton. The directions of the causal relations are determined by the order $\pi=(\pi_1,\dots,\pi_p)$ that ranks the nodes, where $\pi_j$ is the node appearing in position $j$. Given $\pi$, we define the order-consistency mask $M_{jk}(\pi)=1[j\prec_\pi k]$, where $j\prec_\pi k$ means that node $j$ appears before node $k$ in $\pi$. The DAG adjacency matrix is then
\begin{equation*}
G=f(A,\pi):=A\odot M(\pi),
\end{equation*}
where $\odot$ denotes elementwise multiplication. Thus, each skeleton edge is oriented from the earlier node in $\pi$ to the later node. Since all directed edges respect $\pi$, the matrix $G$ is acyclic by construction.

Proposition~\ref{prop:representation} formalizes the exactness of the factorization: any DAG can be obtained from its skeleton and any compatible topological order, and any pair that recovers the DAG must be of this form.
\begin{proposition}
\label{prop:representation}
Let $G^\star$ be any DAG on $p$ nodes, and let $A^\star$ denote its skeleton. Then $G^\star=f(A^\star,\pi)$ for every $\pi\in\operatorname{Top}(G^\star)$. Conversely, if $f(A,\pi)=G^\star$, then $A=A^\star$ and $\pi\in\operatorname{Top}(G^\star)$. Hence
\begin{equation*}
\{(A,\pi):f(A,\pi)=G^\star\}=\{(A^\star,\pi):\pi\in\operatorname{Top}(G^\star)\}.
\end{equation*}
\end{proposition}
See Appendix~\ref{app:representation} for the proof. Here, $\operatorname{Top}(G)$ is the set of topological orders consistent with $G$.

\subsection{Induced conditional distribution}

We now define the conditional distribution $p_\theta(G\mid X)$ over DAGs induced by the skeleton--order factorization, where $X\in\mathbb{R}^{n\times p}$ denotes the input dataset and $\theta$ denotes the model parameters. \texttt{Arrow} parameterizes the skeleton and order components separately by assigning an edge probability to each unordered pair of nodes and a score to each node. For the skeleton, we use a product-Bernoulli distribution with edge probabilities $\nu_{jk}\equiv\nu_{jk}^\theta(X)\in(0,1)$ satisfying $\nu_{jk}=\nu_{kj}$ for $j\neq k$:
\begin{equation*}
p_\theta(A\mid X)=\prod_{j<k}\nu_{jk}^{A_{jk}}(1-\nu_{jk})^{1-A_{jk}}.
\end{equation*}
Thus, the edge probability $\nu_{jk}=\operatorname{Pr}_\theta(A_{jk}=1\mid X)$. For the order, we use a Plackett--Luce distribution \citep{Luce1959,Plackett1975} with node scores $s_j\equiv s_j^\theta(X)\in\mathbb{R}$ for $j=1,\dots,p$:
\begin{equation*}
p_\theta(\pi\mid X)=\prod_{j=1}^p\frac{\exp(s_{\pi_j})}{\sum_{k=j}^p\exp(s_{\pi_k})}.
\end{equation*}
Larger scores favor earlier positions, so $s_j>s_k$ makes node $j$ more likely to precede node $k$. We then form the joint conditional probability distribution over the skeleton and order as
\begin{equation*}
p_\theta(A,\pi \mid X)=p_\theta(A \mid X)p_\theta(\pi \mid X).
\end{equation*}
The deterministic construction $G=f(A,\pi)$ then induces a conditional distribution over DAGs:
\begin{equation*}
p_\theta(G\mid X)=\sum_{A,\pi}1[f(A,\pi)=G]p_\theta(A\mid X)p_\theta(\pi\mid X).
\end{equation*}
The distribution $p_\theta(G\mid X)$ is supported only on DAGs, because each skeleton is oriented according to a single order, so cycles cannot occur. Moreover, since $0<\nu_{jk}<1$, the support is the entire DAG space, because every DAG can be obtained from its skeleton and a compatible topological order.

\subsection{Exact and composite likelihoods}

For the dataset $X$, let $G^\star$ denote the corresponding ground-truth DAG, and let $A^\star$ be the skeleton of $G^\star$. Under the conditional probability model above, the exact likelihood assigned to the DAG $G^\star$ is
\begin{equation*}
p_\theta(G^\star\mid X)=\sum_{A,\pi}1[f(A,\pi)=G^\star]p_\theta(A\mid X)p_\theta(\pi\mid X).
\end{equation*}
By Proposition~\ref{prop:representation}, the event $f(A,\pi)=G^\star$ occurs exactly when $A=A^\star$ and $\pi\in\operatorname{Top}(G^\star)$, so
\begin{equation*}
p_\theta(G^\star\mid X)=p_\theta(A^\star\mid X)\sum_{\pi\in\operatorname{Top}(G^\star)}p_\theta(\pi\mid X).
\end{equation*}
Substituting the product-Bernoulli distribution for the skeleton gives an exact likelihood of the form
\begin{equation*}
p_\theta(G^\star \mid X)=\left[\prod_{j<k}\nu_{jk}^{A^\star_{jk}}(1-\nu_{jk})^{1-A^\star_{jk}}\right]\left[\sum_{\pi\in\operatorname{Top}(G^\star)}p_\theta(\pi\mid X)\right].
\end{equation*}
The second factor makes this likelihood intractable to train with in general, since it requires summing the Plackett--Luce probability mass over all topological orders consistent with the target DAG, a combinatorially large set whose cardinality is itself hard to count \citep{Brightwell1991}.

To avoid the intractable exact likelihood, we adopt a composite likelihood \citep[see, e.g.,][]{Varin2011} for training derived from directed edge marginals. For an ordered pair $(j,k)$ with $j\neq k$, write
\begin{equation*}
q_{jk}=\operatorname{Pr}_\theta(j\prec_\pi k\mid X)=\frac{\exp(s_j)}{\exp(s_j)+\exp(s_k)}=\sigma(s_j-s_k),
\end{equation*}
where $\sigma(\cdot)$ denotes the logistic sigmoid. These precedence probabilities satisfy $q_{jk}+q_{kj}=1$. Since $G=A\odot M(\pi)$, we have $G_{jk}=A_{jk}M_{jk}(\pi)$ and hence the probability of the directed edge $G_{jk}$ is
\begin{equation*}
\operatorname{Pr}_\theta(G_{jk}=1\mid X)=\operatorname{Pr}_\theta(A_{jk}=1,j\prec_\pi k\mid X)=\nu_{jk}q_{jk},
\end{equation*}
where the second equality uses the conditional independence of $A$ and $\pi$ given $X$. The composite likelihood now follows by taking the product of these directed edge marginals across ordered pairs:
\begin{equation*}
\tilde{p}_\theta(G^\star\mid X)=\prod_{j\neq k}\operatorname{Pr}_\theta(G_{jk}=G^\star_{jk}\mid X)=\prod_{j\neq k}(\nu_{jk}q_{jk})^{G^\star_{jk}}(1-\nu_{jk}q_{jk})^{1-G^\star_{jk}}.
\end{equation*}
Taking negative logs of the composite likelihood yields the negative log-likelihood used for training:
\begin{equation}
\label{eq:nll}
\mathcal{L}(\theta)=-\sum_{j\neq k}\left[G^\star_{jk}\log(\nu_{jk}q_{jk})+(1-G^\star_{jk})\log(1-\nu_{jk}q_{jk})\right].
\end{equation}
This composite negative log-likelihood replaces the exact version with lower-dimensional terms.

Theorem~\ref{thm:edge-marginal-identifiability} shows that the composite objective is not merely an edge-classification surrogate. Within the skeleton--order model class, the directed edge marginals identify both the product-Bernoulli skeleton distribution and the Plackett--Luce order distribution, and hence the full latent distribution.
\begin{theorem}
\label{thm:edge-marginal-identifiability}
Fix $X$ and suppose the directed edge marginals are generated by the factorization $r_{jk}:=\Pr_\theta(G_{jk}=1\mid X)=\nu_{jk}\sigma(s_j-s_k)$ for $j\neq k$, where $\nu_{jk}=\nu_{kj}\in(0,1)$. Then the directed edge marginals identify the skeleton edge probabilities and the order-score differences: $\nu_{jk}=r_{jk}+r_{kj}$ and $s_j-s_k=\log(r_{jk}/r_{kj})$. Consequently, the skeleton distribution $p_\theta(A\mid X)$ and the order distribution $p_\theta(\pi\mid X)$ are identified, with the scores $s_1,\dots,s_p$ identified up to an additive constant. Hence, the full latent distribution $p_\theta(A,\pi\mid X)=p_\theta(A\mid X)p_\theta(\pi\mid X)$ is identified.
\end{theorem}
The proof and further details are in Appendices~\ref{app:composite-theory} and \ref{app:composite-interpretation}. This result shows that the composite likelihood is structure-preserving within the skeleton--order family. In a general DAG model, pairwise edge marginals need not determine the joint graph distribution or global structure, since acyclicity induces higher-order dependencies. In contrast, the two directed edge marginals for each unordered pair encode undirected connectivity through their sum and relative order through their log-ratio. Thus, it retains the information needed to recover the full latent distribution under correct specification.

Importantly, Theorem~\ref{thm:edge-marginal-identifiability} establishes model-class identifiability, not observational causal identifiability. Recovering the true causal DAG from observational data still requires standard assumptions, including causal sufficiency, faithfulness, and additional conditions required to distinguish Markov-equivalent DAGs. Under misspecification, the composite likelihood targets the closest directed edge marginals representable by the skeleton--order model. Appendix~\ref{app:misspecification} shows that representability holds if and only if the pairwise directional preferences are globally consistent with a single latent ordering.

\subsection{Maximum-probability prediction}

At inference time, we predict the most probable skeleton--order pair $(\hat{A},\hat{\pi})$ under the joint distribution. Since the joint distribution satisfies $p_\theta(A,\pi\mid X)=p_\theta(A\mid X)p_\theta(\pi\mid X)$, it is maximized at
\begin{equation*}
(\hat{A},\hat{\pi})=\arg\max_{A,\pi}p_\theta(A,\pi\mid X)=\left(\arg\max_A p_\theta(A\mid X),\arg\max_\pi p_\theta(\pi\mid X)\right).
\end{equation*}
For the product-Bernoulli distribution, the most likely skeleton thresholds edge probabilities at 0.5 as $\hat{A}_{jk}=1[\nu_{jk}>0.5]$ for $j\neq k$ and $\hat{A}_{jj}=0$. For the Plackett--Luce distribution, the most likely order sorts node scores in descending order as $s_{\hat{\pi}_1}\geq s_{\hat{\pi}_2}\geq \cdots \geq s_{\hat{\pi}_p}$. Applying the order-consistency mask $M(\hat{\pi})$, where $M_{jk}(\hat{\pi})=1[j\prec_{\hat{\pi}} k]$, gives the prediction $\hat{G}=f(\hat{A},\hat{\pi})=\hat{A}\odot M(\hat{\pi})$. Thus, inference consists of thresholding the undirected edge probabilities to obtain the skeleton $\hat{A}$, sorting the node scores to obtain the order $\hat{\pi}$, and orienting $\hat{A}$ by $\hat{\pi}$ to obtain the final DAG $\hat{G}$. If thresholding recovers the true skeleton and the score order respects every true directed edge, then $\hat{A}=A^\star$ and $\hat{\pi}\in\operatorname{Top}(G^\star)$, and the maximum-probability prediction is exactly the true graph, i.e., $\hat{G}=G^\star$.

\section{Model architecture}
\label{sec:model}

\subsection{Overall structure}

We now describe \texttt{Arrow}'s architecture, which parameterizes the skeleton and order distributions. Given a dataset $X$, we compute one contextualized embedding $h_j\in\mathbb{R}^d$ for each variable $j=1,\dots,p$ in four stages: (1) projection of the scalar dataset entries, (2) interaction across variables within each observation, (3) aggregation across observations within each variable, and (4) global contextualization across the resulting variable embeddings. The latter three stages are implemented as transformer modules \citep{Vaswani2017}. Two prediction heads then map the variable embeddings to the quantities needed for graph prediction: a skeleton head produces the undirected edge probabilities $\nu_{jk}$, and an order head produces the node scores $s_j$. Figure~\ref{fig:arrow_architecture} graphically depicts the full architecture. Exact configuration details and the computational complexity of the architecture are given in Appendix~\ref{app:model}.

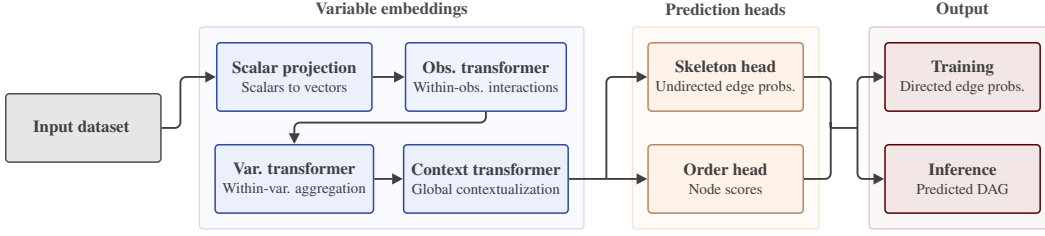
\begin{figure}[t]
\centering
\resizebox{\linewidth}{!}{\begin{tikzpicture}[
    font=\small,
    >={Triangle[length=2mm,width=2mm]},
    flow/.style={->, line width=1.0pt, draw=black!75, rounded corners=3pt},
    block/.style={
        draw,
        thick,
        rounded corners=2pt,
        align=center,
        inner sep=4pt,
        minimum height=1.28cm,
        text=black!75
    },
    group/.style={
        draw,
        rounded corners=2pt,
        inner sep=8pt
    }
]

\node[block, minimum width=2.9cm, fill=RoyalBlue!10, draw=RoyalBlue!70!black] (proj)
{
    \textbf{Scalar projection} \\
    \footnotesize Scalars to vectors
};

\node[block, minimum width=2.9cm, fill=RoyalBlue!8, draw=RoyalBlue!70!black, right=6mm of proj] (obs)
{
    \textbf{Obs. transformer} \\
    \footnotesize Within-obs. interactions
};

\node[block, minimum width=2.9cm, fill=RoyalBlue!10, draw=RoyalBlue!70!black, below=6mm of proj] (var)
{
    \textbf{Var. transformer} \\
    \footnotesize Within-var. aggregation
};

\node[block, minimum width=2.9cm, fill=RoyalBlue!10, draw=RoyalBlue!70!black, below=6mm of obs] (context)
{
    \textbf{Context transformer} \\
    \footnotesize Global contextualization
};

\coordinate (encW) at ($(proj.west)!0.5!(var.west)$);

\node[block, minimum width=2.9cm, fill=Black!10, draw=Black!70, anchor=east] (data) at ($(encW)+(-1cm,0)$)
{
    \textbf{Input dataset}
};

\node[block, minimum width=2.9cm, fill=Peach!10, draw=Peach!70!black, right=1.5cm of obs] (skel)
{
    \textbf{Skeleton head} \\
    \footnotesize Undirected edge probs.
};

\node[block, minimum width=2.9cm, fill=Peach!10, draw=Peach!70!black, below=6mm of skel] (order)
{
    \textbf{Order head} \\
    \footnotesize Node scores
};

\node[block, minimum width=2.9cm, fill=Maroon!10, draw=Maroon!70!black, right=1.5cm of skel] (train)
{
    \textbf{Training} \\
    \footnotesize Directed edge probs.
};

\node[block, minimum width=2.9cm, fill=Maroon!10, draw=Maroon!70!black, below=6mm of train] (infer)
{
    \textbf{Inference} \\
    \footnotesize Predicted DAG
};

\draw[flow] (data.east) -- ++(3.5mm,0) |- (proj.west);
\draw[flow] (proj) -- (obs);
\draw[flow] (obs.south) -- ++(0,-2mm) coordinate (obsbend) -- (obsbend -| var.center) -- (var.north);
\draw[flow] (var) -- (context);
\coordinate (split) at ($(context.east)+(0.7,0)$);
\draw[line width=1.0pt, draw=black!75] (context.east) -- (split);
\draw[flow] (split) |- (skel.west);
\draw[->, line width=1.0pt, draw=black!75] (split) |- (order.west);
\coordinate (midmerge) at ($($(skel.east)!0.5!(order.east)$)+(5mm,0)$);
\coordinate (phaseSplit) at ($(midmerge)+(5mm,0)$);
\draw[line width=1.0pt, draw=black!75, rounded corners=3pt] (skel.east) -- ++(5mm,0) |- (midmerge);
\draw[line width=1.0pt, draw=black!75, rounded corners=3pt] (order.east) -- ++(5mm,0) |- (midmerge);
\draw[line width=1.0pt, draw=black!75] (midmerge) -- (phaseSplit);
\draw[flow] (phaseSplit) |- (train.west);
\draw[flow] (phaseSplit) |- (infer.west);

\begin{pgfonlayer}{background}
    \node[group, fit=(proj)(obs)(var)(context), fill=RoyalBlue!3, draw=RoyalBlue!30] (encbox) {};
    \node[group, fit=(skel)(order), fill=Peach!3, draw=Peach!30] (headbox) {};
    \node[group, fit=(train)(infer), fill=Maroon!3, draw=Maroon!30] (phasebox) {};
\end{pgfonlayer}

\node[font=\bfseries\small, text=black!70]
    at ($(encbox.north)+(0,0.32)$) {Variable embeddings};

\node[font=\bfseries\small, text=black!70]
    at ($(headbox.north)+(0,0.32)$) {Prediction heads};

\node[font=\bfseries\small, text=black!70]
    at ($(phasebox.north)+(0,0.32)$) {Output};

\end{tikzpicture}}
\caption{Architecture of \texttt{Arrow}. The model embeds an input dataset via scalar projection and three transformer stages for within-observation variable interaction, within-variable observation aggregation, and global contextualization. Skeleton and order heads predict undirected edge probabilities and node scores, yielding directed edge probabilities for training and predicted DAGs for inference.}
\label{fig:arrow_architecture}
\end{figure}

\texttt{Arrow} respects the exchangeability of i.i.d.\ tabular data because the architecture uses no positional encodings, and attention without positional encodings is permutation equivariant \citep{Lee2019}. In particular, because it aggregates over observations while retaining variable-indexed representations, its output is invariant to observation order and equivariant to variable order. Thus, permuting observations leaves the output unchanged, while permuting variables relabels the outputs accordingly.

\subsection{Variable embeddings}

Given $X$, the encoder eventually produces one embedding $h_j$ for each variable $j$. The first stage is a projection that maps each scalar dataset entry $x_{ij}$ to the model dimension using a shared linear layer:
\begin{equation*}
z_{ij}=W_{\mathrm{proj}}x_{ij}+b_{\mathrm{proj}},\qquad z_{ij}\in\mathbb{R}^d.
\end{equation*}
For each observation $i$, we collect the projected entries into a sequence $Z_i=[z_{i1},\dots,z_{ip}]\in\mathbb{R}^{p\times d}$. The observation transformer $T_{\mathrm{obs}}$ then applies self-attention across variables within this sequence:
\begin{equation*}
R_i=T_{\mathrm{obs}}(Z_i)\in\mathbb{R}^{p\times d}.
\end{equation*}
This stage allows each variable representation within an observation to interact with the other variables in the same observation. The variable transformer $T_{\mathrm{var}}$ then aggregates observations for each variable. For variable $j$, let $R_{:j}=[r_{1j},\dots,r_{nj}]\in\mathbb{R}^{n\times d}$ denote its observation-wise representations. We introduce $m$ learned summary tokens $S\in\mathbb{R}^{m\times d}$ and update them by cross-attending to $R_{:j}$:
\begin{equation*}
C_j=T_{\mathrm{var}}(S,R_{:j})\in\mathbb{R}^{m\times d}.
\end{equation*}
The summary-token queries have length $m$, so cross-attention to the $n$ observation-wise representations costs $O(mn)$ per variable, avoiding the $O(n^2)$ cost of full self-attention across observations. The updated summary tokens are concatenated and merged to obtain a single vector for variable $j$:
\begin{equation*}
u_j=W_{\mathrm{merge}}\mathrm{vec}(C_j)+b_{\mathrm{merge}},\qquad u_j\in\mathbb{R}^d.
\end{equation*}
The merged representation has fixed dimension independent of the number of observations $n$. Finally, after observation aggregation, the context transformer is applied across the variable representations:
\begin{equation*}
[h_1,\dots,h_p]=T_{\mathrm{context}}([u_1,\dots,u_p])\in\mathbb{R}^{p\times d}.
\end{equation*}
This last stage lets each embedding incorporate information from all other embeddings, producing the final contextualized variable embeddings used by the skeleton head and the order head.

\subsection{Prediction heads}

Given the contextualized variable embeddings $h_1,\dots,h_p$, two prediction heads produce the quantities required by the factorized graph distribution: skeleton probabilities $\nu_{jk}$ and order scores $s_j$. The skeleton head is implemented as an MLP $g_{\mathrm{skel}}:\mathbb{R}^{2d}\to\mathbb{R}$ applied to pairs of embeddings, while the order head is implemented as a linear map $g_{\mathrm{ord}}:\mathbb{R}^d\to\mathbb{R}$ applied to individual embeddings. Because the skeleton is undirected, we require $\nu_{jk}=\nu_{kj}$ and hence use Janossy pooling \citep{Murphy2019} over the two possible orderings of each pair. For each ordered pair $(j,k)$ with $j\neq k$, we form the concatenated representation $[h_j \,\|\, h_k]\in\mathbb{R}^{2d}$ and compute the undirected edge probability $\nu_{jk}$ as
\begin{equation*}
\nu_{jk}=\sigma(e_{jk}),\qquad e_{jk}=\frac{1}{2}\left(\tilde{e}_{jk}+\tilde{e}_{kj}\right),\qquad \tilde{e}_{jk}=g_{\mathrm{skel}}([h_j\,\|\,h_k]).
\end{equation*}
This formulation is order invariant because it averages over the two pairs, so the undirected edge probabilities $\nu_{jk}=\nu_{kj}$. For each variable $j$, the order head computes the node score as
$s_j=g_{\mathrm{ord}}(h_j)$.

\section{Pretraining scheme}
\label{sec:pretraining}

\subsection{Data generation}
\label{sec:data-generation}

\texttt{Arrow}'s training data constitutes synthetic causal tasks of varying graph family, functional form, noise model, and dataset shape. For each task, we sample the number of observations and variables as $n\sim\operatorname{Unif}\{100,\dots,2000\}$ and $p\sim\operatorname{Unif}\{2,\dots,100\}$. Next, we choose between the Erdős--Rényi \citep{Erdos1959} and scale-free \citep{Barabasi1999} families of graphs with equal probability, and then sample the number of edges $s$ uniformly from $\{0,\dots,4p\}$, allowing edge density to vary across tasks. The Erdős--Rényi model samples an approximately homogeneous random graph, while the scale-free model samples a static power-law graph with more heterogeneous degrees. To transform these undirected graphs into DAGs, we sample a random permutation of the nodes and orient every undirected edge from the earlier node to the later node in that permutation.

Conditioned on the sampled DAG $G^\star$, observations are generated by sampling a structural equation model \citep[SEM,][]{Pearl2009}. The SEM family is chosen as a linear model or a nonlinear MLP model with equal probability. Variables are generated in topological order and standardized immediately after generation, which stabilizes the data-generating process and avoids sortability artifacts that can otherwise arise in synthetic data \citep{Ormaniec2025}. In the linear case, the SEM is of the form $x_j=w_j^\top x+\varepsilon_j$, where $w_j\in\mathbb{R}^p$ is a sparse weight vector with nonzero entries only at coordinates corresponding to node $j$'s parents. In the nonlinear case, the SEM is of the form $x_j=c_j^\top \phi(W_j^\top x)+\varepsilon_j$, where $c_j\in\mathbb{R}^h$ is the output-layer weight vector and $W_j\in\mathbb{R}^{p\times h}$ is a sparse hidden-layer weight matrix with nonzero rows only at coordinates corresponding to parents of node $j$. The hidden dimension $h$ is sampled from $\{1,\dots,64\}$ and the activation function $\phi$ is sampled from the set $\{\tanh,\mathrm{hardtanh},\mathrm{sigmoid},\mathrm{hardsigmoid}\}$. In both SEMs, weights are sampled independently with random signs and magnitudes drawn from $[1-a_w,1+a_w]$, where $a_w \sim \operatorname{Unif}(0,0.9)$.

The noise model is also randomized. With equal probability, we use either a homogeneous setting, where all nodes share the same noise distribution, or a heterogeneous setting, where nodes can have different noise distributions. Noise distributions are sampled from $\operatorname{Norm}(0,1)$, $\operatorname{Unif}(-1,1)$, and $\operatorname{Beta}(\alpha,\beta)$, where $\alpha,\beta\sim\operatorname{Unif}(1,10)$. These choices cover canonical unbounded and bounded noise families, with the beta family adding flexible asymmetry. To vary signal-to-noise ratios, we sample a task-level beta distribution with shape parameters drawn uniformly from $[1,10]$. For each non-root node, we sample a target local $R^2$ in $[0.1,0.9]$ from this distribution and choose a scaling factor so that the variance of the signal relative to the variance of the scaled noise attains this target.

\subsection{Optimization}

\texttt{Arrow} is trained end-to-end by minimizing the composite negative log-likelihood in Section~\ref{sec:factored} with AdamW \citep{Loshchilov2019}. The optimizer is run for $250{,}000$ iterations with a batch size of $2048$. Because the task generator is effectively unlimited, we do not train on a fixed collection of simulated datasets. Instead, each iteration samples a dataset shape $(n,p)$ and draws a fresh batch of newly generated tasks from the generator. This approach prioritizes coverage of the task distribution and targets the expected loss under the generator, consistent with stochastic optimization and large-scale training practice \citep{Bottou2018,Kaplan2020,Hoffmann2022}. Appendix~\ref{app:pretraining} elaborates on this task-streaming pretraining strategy and provides additional computational details.

\section{Related work}

The prior work most relevant to \texttt{Arrow} falls into two broad categories, distinguished by how they handle new datasets. Task-specific methods solve a fresh causal discovery problem from scratch, while pretrained models use a reusable predictor learned from many different causal discovery tasks.

Among task-specific methods, recent work formulates causal discovery as continuous optimization over a weighted adjacency matrix subject to an acyclicity constraint. \citet{Zheng2018} introduced a smooth differentiable characterization of acyclicity. Later work adapted this framework to nonlinear models \citep{Yu2019,Zheng2020}, likelihood-based objectives \citep{Ng2020}, and alternative acyclicity formulations, including \texttt{DAGMA} \citep{Bello2022} and \texttt{SDCD} \citep{Nazaret2024}. Other task-specific approaches include score- and order-based methods that search over graphs or topological orders, such as \citet{Andrews2023} and \citet{Duong2025}, as well as test-based approaches that infer adjacencies directly \citep{Amin2024}. A related line of task-specific methods learns posterior distributions over graphs \citep{Lorch2021,Charpentier2022,Bonilla2024,Thompson2025}. Although task-specific methods have driven much of the progress in the field, they must be re-solved on every dataset, which can be computationally costly and prevents knowledge accumulated across tasks from transferring to new datasets.

In terms of pretrained models, \texttt{AVICI} \citep{Lorch2022} and related early approaches \citep{Ke2023,Petersen2023} showed that neural predictors can produce causal graph predictions directly from data, though with limited pretraining diversity and without acyclicity guarantees. Subsequent work expanded the setting in complementary directions: \texttt{BCNP} uses Bayesian meta-learning with posterior sampling over DAGs \citep{Dhir2025}, \texttt{SEA} aggregates task-specific method estimates from sampled variable subsets \citep{Wu2025}, \texttt{ADAG} uses unsupervised pretraining for linear SEMs \citep{Yin2025}, and \texttt{CauScale} scales to larger graphs without explicitly enforcing acyclicity \citep{Peng2026}. Other pretrained approaches incorporate auxiliary graph knowledge \citep{Xu2026} or adapt frozen tabular foundation model embeddings \citep{Swelam2025}. As a collective, these works provide early signs of promise, but still do not jointly address zero-shot inference, guaranteed DAG outputs, and broad pretraining diversity. \texttt{Arrow} addresses this gap through zero-shot graph prediction, a skeleton--order-induced DAG guarantee, and diverse supervised pretraining.

\section{Experiments}
\label{sec:experiments}

\subsection{Baselines and metrics}

We evaluate \texttt{Arrow} against \texttt{AVICI} \citep{Lorch2022}, \texttt{BCNP} \citep{Dhir2025}, \texttt{SEA} \citep{Wu2025}, \texttt{DAGMA} \citep{Bello2022}, and \texttt{SDCD} \citep{Nazaret2024}. The first three baselines are pretrained models, while the latter two are strong task-specific methods. We use the publicly released checkpoints for the pretrained models. We do not include \texttt{ADAG} \citep{Yin2025} or \texttt{CauScale} \citep{Peng2026} in these evaluations, as neither currently has public code or checkpoints available. As performance metrics, we report the normalized structural Hamming distance (nSHD), F1 score (F1), average precision (AP), and runtime in seconds. nSHD is SHD normalized by the SHD between the empty graph and the ground-truth graph \citep{Yang2022}, making it comparable across graphs with different ground-truth edge counts. Appendix~\ref{app:setup} details the compute setup and software sources.

\subsection{In-distribution data}
\label{sec:in_distribution}

We first evaluate on newly generated synthetic datasets sampled from the data-generating distribution used in pretraining, which is representative of standard synthetic designs in the causal discovery literature. Here, we study how performance changes with the dataset shape. In each setting, we draw 100 datasets, with the graph family, functional form, and noise model randomized across datasets. Appendix~\ref{app:in_distribution} presents additional experiments that vary the other generator components individually.

Figure~\ref{fig:synthetic} reports two sets of results, one varying the number of observations with the number of variables randomized, and one varying the number of variables with the number of observations randomized. These results indicate that \texttt{Arrow} maintains strong performance across dataset shapes. Increasing the number of observations improves graph recovery, with nSHD decreasing and F1 and AP increasing before the gains begin to saturate at the largest sample sizes. Increasing the number of variables naturally makes graph recovery harder, but \texttt{Arrow} remains the strongest method overall and is consistently best on nSHD, F1, and AP. \texttt{SEA} is the closest competitor in several cases, particularly for nSHD, but \texttt{Arrow} is roughly 100--1000× faster. Though both models are pretrained, \texttt{Arrow} predicts graphs zero-shot, whereas \texttt{SEA} relies on outputs from task-specific methods for inference.

\begin{figure}[ht]
\centering
\includegraphics[width=\linewidth]{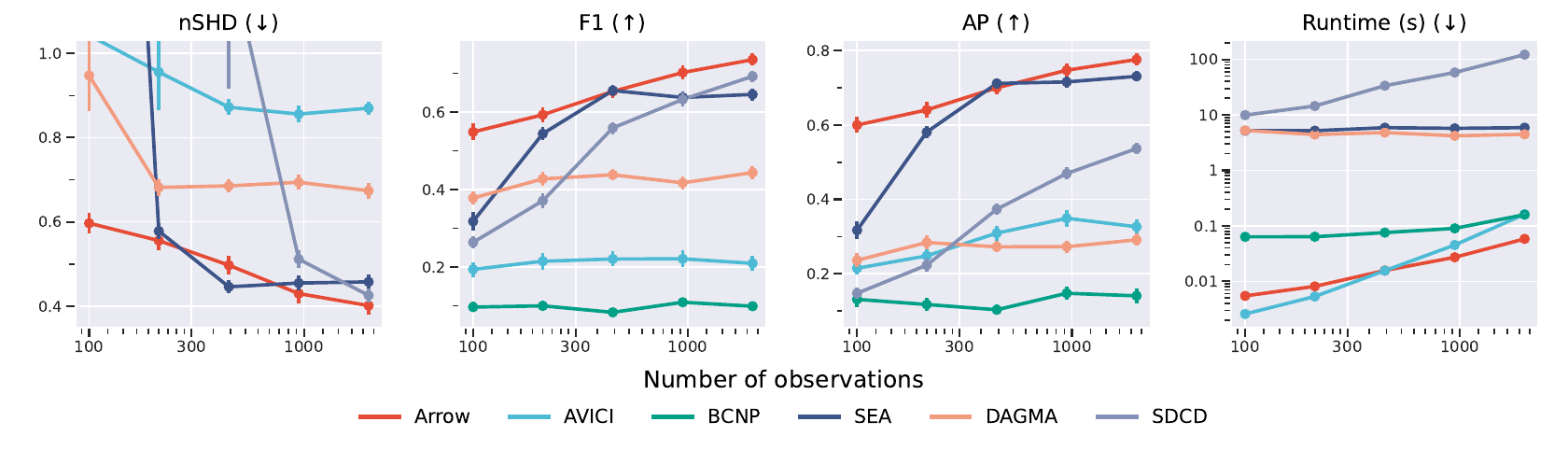}
\includegraphics[width=\linewidth]{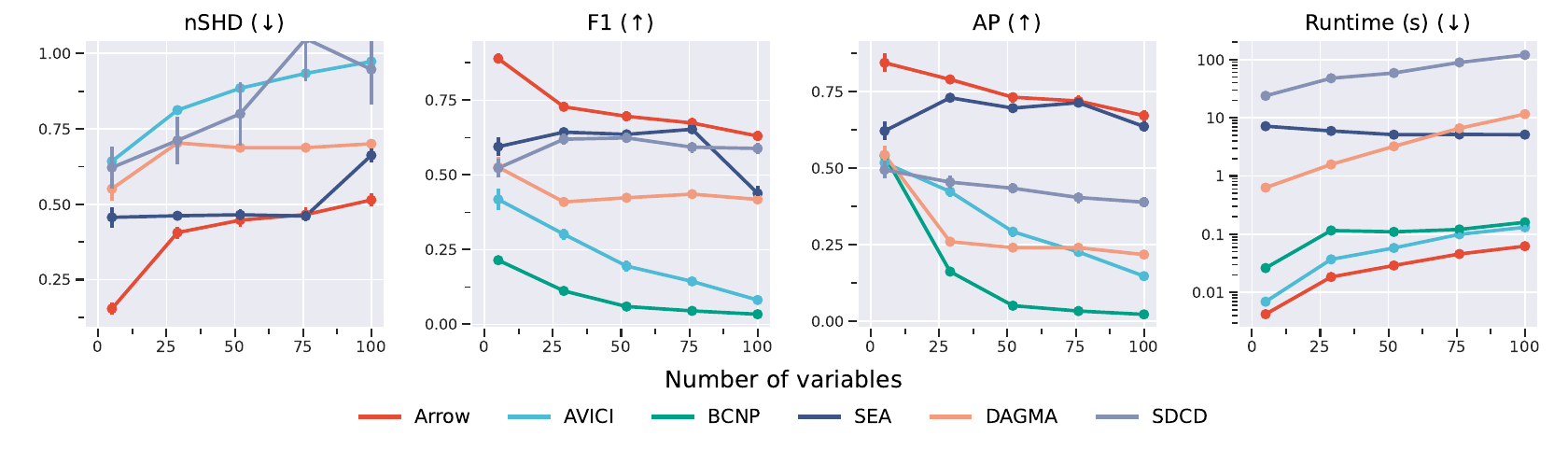}
\caption{Performance as a function of the number of observations and variables on in-distribution synthetic datasets. The averages (solid points) and standard errors (error bars) summarize results over 100 datasets. The nSHD axis is truncated at 1 to preserve resolution among stronger baselines.}
\label{fig:synthetic}
\end{figure}

\subsection{Out-of-distribution data}

To test out-of-distribution generalization in a controlled synthetic setting, we modify the data generator along three core axes: graph family, functional form, and noise model, replacing the in-distribution choices with small-world graphs \citep{Watts1998}, spline functions, and gamma noise, respectively. These shifts introduce different types of structure unseen during training. Figure~\ref{fig:synthetic_ood} shows strong performance under the shifts in the graph family and noise model, with only modest degradation relative to no shift, while the shift in functional form produces the most noticeable change. Under the simultaneous shift, \texttt{Arrow} continues to produce relatively strong graph predictions with sub-second inference, indicating robustness beyond the exact data generator used for pretraining.

\begin{figure}[ht]
\centering
\includegraphics[width=\linewidth]{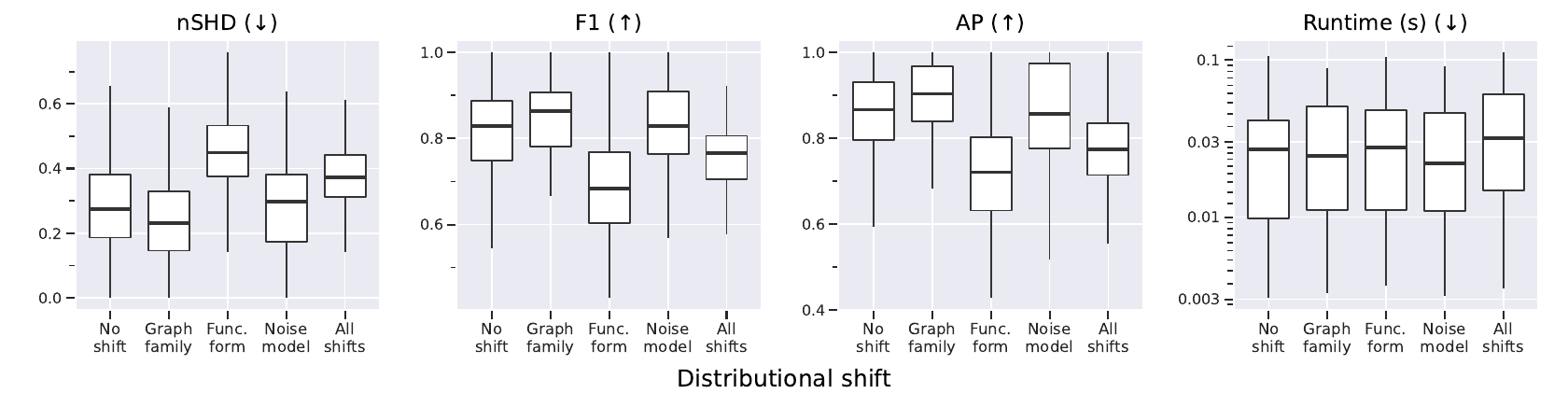}
\caption{Performance as a function of distribution shifts on synthetic datasets. Each shift modifies one aspect of the data-generating process relative to the training distribution by replacing the graph family, functional form, or noise model. The boxplots summarize results over 100 datasets.}
\label{fig:synthetic_ood}
\end{figure}

Beyond the fully synthetic setting, we evaluate on semi-synthetic datasets generated from four real-world networks in the Bayesian Network Repository, with sizes ranging from $p=46$ to $p=107$ variables and $s=66$ to $s=150$ edges. Table~\ref{tab:bayesian} shows that \texttt{Arrow} generalizes well in this setting, attaining the best nSHD and AP on all four networks. \texttt{Arrow} is also competitive in F1, with the best result on ECOLI70 and the larger ARTH150, and a near-best result on MAGIC-IRRI. Across all networks, \texttt{Arrow} remains substantially faster than task-specific methods and also faster than some pretrained models. These results suggest that the performance gains observed above carry over to semi-synthetic benchmarks with real-world graph structures and independently specified generators.

Appendix~\ref{app:out_of_distribution} reports results on the real flow cytometry dataset of \citet{Sachs2005} and the synthetic benchmark datasets of \citet{Herman2025}. These experiments complement the main out-of-distribution evaluations by testing \texttt{Arrow} on benchmarks with distinct distributional characteristics, showing that it remains competitive beyond the primary synthetic and semi-synthetic settings.

\begin{table}[ht]
\centering
\caption{Performance on out-of-distribution semi-synthetic datasets from the Bayesian Network Repository. The averages and standard errors summarize results over 100 datasets. Bold indicates the best result per metric, including ties whose standard error intervals overlap with the best entry.}
\label{tab:bayesian}
\scriptsize
\begingroup
\setlength{\tabcolsep}{3.6pt}
\begin{minipage}[t]{0.49\textwidth}
\centering
\vspace{0.3em}ECOLI70 ($n=100$, $p=46$, $s=70$) \par\vspace{0.3em}
\begin{tabular}{lp{0.47in}p{0.47in}p{0.47in}p{0.47in}}
\toprule
 & nSHD (↓) & F1 (↑) & AP (↑) & Time (↓) \\
\midrule
\texttt{Arrow} & \textbf{0.59 (0.00)} & \textbf{0.56 (0.00)} & \textbf{0.56 (0.00)} & \textbf{0.01 (0.00)} \\
\texttt{AVICI} & 0.73 (0.01) & 0.46 (0.01) & 0.47 (0.01) & \textbf{0.01 (0.00)} \\
\texttt{BCNP}  & 3.61 (0.03) & 0.10 (0.00) & 0.12 (0.00) & 0.06 (0.00) \\
\texttt{SEA}   & 2.23 (0.02) & 0.27 (0.00) & 0.24 (0.00) & 6.59 (0.01) \\
\texttt{DAGMA} & 0.70 (0.01) & 0.49 (0.00) & 0.27 (0.00) & 4.66 (0.06) \\
\texttt{SDCD}  & 2.13 (0.02) & 0.33 (0.00) & 0.15 (0.00) & 11.22 (0.10) \\
\bottomrule
\end{tabular}
\end{minipage}
\hfill
\begin{minipage}[t]{0.48\textwidth}
\centering
\vspace{0.3em}MAGIC-NIAB ($n=100$, $p=44$, $s=66$) \par\vspace{0.3em}
\begin{tabular}{lp{0.47in}p{0.47in}p{0.47in}p{0.47in}}
\toprule
 & nSHD (↓) & F1 (↑) & AP (↑) & Time (↓) \\
\midrule
\texttt{Arrow} & \textbf{0.98 (0.00)} & 0.08 (0.00) & \textbf{0.18 (0.00)} & \textbf{0.01 (0.00)} \\
\texttt{AVICI} & 1.19 (0.01) & 0.07 (0.00) & 0.07 (0.00) & \textbf{0.01 (0.00)} \\
\texttt{BCNP}  & 2.79 (0.02) & 0.05 (0.00) & 0.04 (0.00) & 0.03 (0.00) \\
\texttt{SEA}   & 1.04 (0.01) & 0.14 (0.00) & 0.15 (0.00) & 4.08 (0.01) \\
\texttt{DAGMA} & 0.99 (0.00) & 0.09 (0.00) & 0.05 (0.00) & 2.30 (0.02) \\
\texttt{SDCD}  & 1.68 (0.04) & \textbf{0.15 (0.00)} & 0.05 (0.00) & 7.99 (0.13) \\
\bottomrule
\end{tabular}
\end{minipage}

\medskip

\begin{minipage}[t]{0.49\textwidth}
\centering
\vspace{0.3em}MAGIC-IRRI ($n=100$, $p=64$, $s=102$) \par\vspace{0.3em}
\begin{tabular}{lp{0.47in}p{0.47in}p{0.47in}p{0.47in}}
\toprule
 & nSHD (↓) & F1 (↑) & AP (↑) & Time (↓) \\
\midrule
\texttt{Arrow} & \textbf{0.92 (0.00)} & 0.14 (0.00) & \textbf{0.23 (0.00)} & \textbf{0.01 (0.00)} \\
\texttt{AVICI} & 1.03 (0.00) & 0.09 (0.00) & 0.08 (0.00) & \textbf{0.01 (0.00)} \\
\texttt{BCNP}  & 3.60 (0.02) & 0.04 (0.00) & 0.03 (0.00) & 0.07 (0.00) \\
\texttt{SEA}   & 1.50 (0.02) & 0.12 (0.00) & 0.08 (0.00) & 4.50 (0.03) \\
\texttt{DAGMA} & 1.01 (0.00) & 0.13 (0.00) & 0.05 (0.00) & 5.19 (0.04) \\
\texttt{SDCD}  & 2.65 (0.04) & \textbf{0.17 (0.00)} & 0.05 (0.00) & 12.09 (0.10) \\
\bottomrule
\end{tabular}
\end{minipage}
\hfill
\begin{minipage}[t]{0.49\textwidth}
\centering
\vspace{0.3em}ARTH150 ($n=100$, $p=107$, $s=150$) \par\vspace{0.3em}
\begin{tabular}{lp{0.47in}p{0.47in}p{0.47in}p{0.47in}}
\toprule
 & nSHD (↓) & F1 (↑) & AP (↑) & Time (↓) \\
\midrule
\texttt{Arrow} & \textbf{0.78 (0.00)} & \textbf{0.37 (0.00)} & \textbf{0.36 (0.00)} & \textbf{0.01 (0.00)} \\
\texttt{AVICI} & 0.94 (0.00) & 0.14 (0.00) & 0.15 (0.00) & \textbf{0.01 (0.00)} \\
\texttt{BCNP}  & 6.39 (0.05) & 0.03 (0.00) & 0.02 (0.00) & 0.07 (0.00) \\
\texttt{SEA}   & 0.99 (0.00) & 0.03 (0.00) & 0.33 (0.00) & 5.19 (0.01) \\
\texttt{DAGMA} & 0.92 (0.01) & \textbf{0.37 (0.00)} & 0.15 (0.00) & 17.62 (0.21) \\
\texttt{SDCD}  & 3.46 (0.05) & 0.21 (0.00) & 0.07 (0.00) & 16.93 (0.15) \\
\bottomrule
\end{tabular}
\end{minipage}
\endgroup
\end{table}

\section{Concluding remarks}
\label{sec:concluding}

Our paper demonstrates that causal discovery can be framed as foundation model pretraining over synthetic observational data, rather than only as a problem that must be re-solved from scratch for each new dataset. The skeleton--order factorization provides a tractable way to enforce acyclicity, and the resulting training objective makes large-scale supervised pretraining practical. Together, these ingredients make it possible to learn a single model that applies across a wide range of graph families, functional forms, noise models, and dataset shapes. This perspective opens several natural directions for future work. The most immediate ones are extending the framework to settings with latent confounding and interventional data, and scaling to larger datasets than those studied here.

\bibliographystyle{plainnat}
\bibliography{library}

\appendix

\section{Details for Table~\ref{tab:capabilities}}
\label{app:capabilities}
In this section, we clarify the criteria for comparing the pretrained models in Table~\ref{tab:capabilities}.

\begin{itemize}
\item \textbf{Zero-shot inference} means graph prediction without task-specific search, training, or optimization, using only a single forward pass of the network.
\item \textbf{DAG guarantee} means that acyclicity is intrinsically guaranteed without postprocessing at inference time. 
\item \textbf{Pretraining diversity} qualitatively rates the breadth of graph families, functional forms, noise models, and dataset shapes in the pretraining data generation. The details are provided in Table~\ref{tab:pretraining_diversity}. A cross indicates that the corresponding axis is fixed or not substantially varied in the reported pretraining setup. It does not, however, imply that the method is deficient for its original objective. Rather, the comparison is designed to identify a remaining gap in the literature: existing pretrained causal discovery models typically cover only a subset of these axes, whereas a broadly reusable zero-shot model benefits from pretraining over all four.
\end{itemize}

\begin{table}[ht]
\centering
\footnotesize
\caption{Comparison of pretraining data diversity in pretrained causal discovery models. We assess whether each method varies four dimensions of its pretraining data-generating process: graph family, functional form, noise model, and dataset shape. The overall diversity rating is determined by the number of varied dimensions: Very high = 4, High = 3, Moderate = 2, and Low = 1. For dataset shape, we consider that a method ticks the box if its pretraining varies both $n$ and $p$. Here we do not consider the diversity of the data that a model does inference on.}
\setlength{\tabcolsep}{4pt}
\label{tab:pretraining_diversity}
\begin{tabular}{@{}lccccc@{}}
\toprule
& Graph family
& Functional form
& Noise model
& Dataset shape
& Rating \\
\midrule
\texttt{AVICI} \citep{Lorch2022}
& \cmark & \cmark & \cmark & \xmark & High \\

\texttt{CSIvA} \citep{Ke2023}
& \xmark & \cmark & \xmark & \xmark & Low \\

\texttt{BCNP} \citep{Dhir2025}
& \xmark & \cmark & \xmark & \xmark & Low \\

\texttt{SEA} \citep{Wu2025}
& \cmark & \cmark & \xmark & \cmark & High \\

\texttt{ADAG} \citep{Yin2025}
& \xmark & \xmark & \xmark & \cmark & Low \\

\texttt{CauScale} \citep{Peng2026}
& \cmark & \cmark & \xmark & \xmark & Moderate \\

\texttt{Arrow} (ours)
& \cmark & \cmark & \cmark & \cmark & Very high \\

\bottomrule
\end{tabular}
\end{table}

Under the definitions of the criteria, the rationale for Table~\ref{tab:capabilities} is as follows.

\begin{itemize}
\item \texttt{AVICI} \citep{Lorch2022}

\xmark{} DAG guarantee: Regularizes/enforces acyclicity during training, but does not produce graphs at inference time through an intrinsically acyclicity-preserving mechanism.
\xmark{} Dataset shape: $n$ is fixed to 200 for training, although $p$ is varied.

\item \texttt{CSIvA} \citep{Ke2023}
\xmark{} Graph family: The pretraining graphs are sampled from a single Erdős--Rényi-style random DAG family.
\xmark{} DAG guarantee: Acyclicity is not enforced.
\xmark{} Noise model: Only Gaussian noise is used for training.
\xmark{} Dataset shape: With $p$ varied, the main experiments fix $n=1500$ for training while there are ablation studies varying sample count.

\item \texttt{BCNP} \citep{Dhir2025}

\xmark{} Graph family: Only Erdős--Rényi graphs are used for training. 
\xmark{} Noise model: Only Gaussian noise with varying scales/hyperparameters is used for training but it does not vary across distinct noise families.
\xmark{} Dataset shape: $n$ is fixed to $1000$ and $p$ is fixed to $20$ for training.

\item \texttt{SEA} \citep{Wu2025}

\xmark{} Zero-shot inference: Not under our definition, because SEA first computes task-specific summary statistics and classical causal discovery estimates on subsets of the target data before aggregation.
\xmark{} DAG guarantee: The inference-time aggregator outputs edge scores and the paper reports that predictions are highly, but not perfectly, acyclic. 
\xmark{} Noise model: Training uses Gaussian noise; non-Gaussian noise is considered as an out-of-distribution evaluation setting rather than as part of the training noise distribution.

\item \texttt{ADAG} \citep{Yin2025}

\xmark{} DAG guarantee: Although acyclicity is enforced through an optimization constraint during training, inference applies the learned map and thresholds the weighted adjacency matrix, which can contain cycles. 
\xmark{} Graph family: Only Erdős--Rényi graphs are used for training.
\xmark{} Functional form: Only linear SEMs are used.
\xmark{} Noise model: The main pretraining/evaluation setup uses Gaussian noise; non-Gaussian noise is considered only in a separate ablation rather than as part of a diverse pretraining mixture.

\item \texttt{CauScale} \citep{Peng2026}

\xmark{} DAG guarantee: Acyclicity is not enforced. 
\xmark{} Noise model: Only Gaussian noise is used for training. 
\xmark{} Dataset shape: With $p$ varied, $n$ is fixed to $1000$ for synthetic data. A separate model is trained on the SERGIO GRN data with $n=5000$ and max $p=200$. Under our criterion, we do not count it as varying dataset shape.

\item \texttt{Arrow}

\cmark{} Zero-shot inference: Inference is a single forward pass of one universal model for all experiments.
\cmark{} DAG guarantee: The skeleton--order factorization orients all selected skeleton edges according to one topological order at inference time.
\cmark{} Graph family, functional form, noise model, and dataset shape: The pretraining distribution explicitly covers multiple graphs, SEM mechanisms, noise distributions, and varying $(n,p)$.
\end{itemize}

\section{Likelihood details}

\subsection{Skeleton--order representation}
\label{app:representation}

\begin{proof}[Proof of Proposition~\ref{prop:representation}]
If $\pi\in\operatorname{Top}(G^\star)$, then every directed edge of $G^\star$ points from an earlier node in $\pi$ to a later one. Since $A^\star$ records exactly the adjacencies of $G^\star$, the masking operation $A^\star\odot M(\pi)$ orients each skeleton edge according to $\pi$ and therefore reproduces $G^\star$. Thus, $f(A^\star,\pi)=G^\star$ for every $\pi\in\operatorname{Top}(G^\star)$.

Conversely, suppose $f(A,\pi)=G^\star$. Because $f(A,\pi)$ is obtained by orienting the undirected edges of $A$ according to $\pi$, the directed graph $f(A,\pi)$ has the same undirected skeleton as $A$. Therefore the skeleton of $G^\star$ must equal $A$, so $A=A^\star$. Moreover, every directed edge of $G^\star$ must point from an earlier node in $\pi$ to a later one, which means $\pi$ is a topological order of $G^\star$. Hence, $\pi\in\operatorname{Top}(G^\star)$.
\end{proof}

\subsection{Composite likelihood theory}
\label{app:composite-theory}

Here, we prove Theorem~\ref{thm:edge-marginal-identifiability} and then characterize the population target of the composite likelihood. The first result shows that, within the skeleton--order model class, directed edge marginals identify the full latent distribution. The second shows that the composite likelihood targets these directed edge marginals at the population level.

We first prove the identifiability result stated in Theorem~\ref{thm:edge-marginal-identifiability}. Recall that, for fixed $X$, the directed edge marginals satisfy
\begin{equation*}
r_{jk}:=\Pr_\theta(G_{jk}=1\mid X)=\nu_{jk}\sigma(s_j-s_k),\qquad j\neq k,
\end{equation*}
with $\nu_{jk}=\nu_{kj}\in(0,1)$.

\begin{proof}[Proof of Theorem~\ref{thm:edge-marginal-identifiability}]
For $j\neq k$, let
\begin{equation*}
q_{jk}=\sigma(s_j-s_k).
\end{equation*}
Using the identity $\sigma(x) + \sigma(-x) = 1$,
\begin{equation*}
q_{jk}+q_{kj}=1,
\end{equation*}
we have
\begin{equation*}
r_{jk}+r_{kj}=\nu_{jk}q_{jk}+\nu_{jk}q_{kj}=\nu_{jk}.
\end{equation*}
Thus, the skeleton probability $\nu_{jk}$ is identified by the sum of the two directed marginals.

For the order scores,
\begin{equation*}
\frac{r_{jk}}{r_{kj}}=\frac{\nu_{jk}q_{jk}}{\nu_{jk}q_{kj}}=\frac{\sigma(s_j-s_k)}{\sigma(s_k-s_j)}.
\end{equation*}
Using
\begin{equation*}
\frac{\sigma(a)}{\sigma(-a)}=\exp(a),
\end{equation*}
with $a=s_j-s_k$, we obtain
\begin{equation*}
\frac{r_{jk}}{r_{kj}}=\exp(s_j-s_k),
\end{equation*}
and therefore
\begin{equation*}
s_j-s_k=\log\left(\frac{r_{jk}}{r_{kj}}\right).
\end{equation*}
The collection of all pairwise differences determines $s$ up to an additive constant. Since the Plackett--Luce distribution is invariant to adding a common constant to all scores, $p_\theta(\pi\mid X)$ is identified. Together with the identified skeleton probabilities $\nu_{jk}$, this identifies $p_\theta(A\mid X)$ and hence the full factorized latent distribution $p_\theta(A,\pi\mid X)$.
\end{proof}

Having established that directed edge marginals identify the latent skeleton and order components within the skeleton--order family, we now characterize the population target of the composite likelihood. For each ordered pair $(j,k)$ with $j\neq k$, define the predicted directed edge marginal
\begin{equation*}
r^\theta_{jk}(X):=\Pr_\theta(G_{jk}=1\mid X)=\nu^\theta_{jk}(X)\sigma\!\left(s^\theta_j(X)-s^\theta_k(X)\right).
\end{equation*}
The composite negative log-likelihood can then be written as
\begin{equation*}
\mathcal{L}(\theta)=-\sum_{j\neq k}\left[G^\star_{jk}\log r^\theta_{jk}(X)+(1-G^\star_{jk})\log(1-r^\theta_{jk}(X))\right].
\end{equation*}

\begin{theorem}[Population target of the composite likelihood]
\label{thm:composite}
Let $(X,G^\star)$ be drawn from a meta-distribution over datasets and DAGs, and define
\begin{equation*}
\eta_{jk}(X):=\Pr(G^\star_{jk}=1\mid X),\qquad j\neq k.
\end{equation*}
For any measurable directed edge predictor $r(X)=\{r_{jk}(X)\}_{j\neq k}$ with $r_{jk}(X)\in(0,1)$, define the population composite risk
\begin{equation*}
\mathcal{R}(r)=\mathbb{E}\left[-\sum_{j\neq k}\left\{G^\star_{jk}\log r_{jk}(X)+(1-G^\star_{jk})\log(1-r_{jk}(X))\right\}\right].
\end{equation*}
Then $\mathcal{R}(r)$ is uniquely minimized, pointwise in $X$, by
\begin{equation*}
r^\star_{jk}(X)=\eta_{jk}(X),\qquad j\neq k.
\end{equation*}
Moreover, for any $r$,
\begin{equation*}
\mathcal{R}(r)-\mathcal{R}(r^\star)=\mathbb{E}\left[\sum_{j\neq k}\operatorname{KL}\left(\operatorname{Bern}(\eta_{jk}(X))\,\middle\|\,\operatorname{Bern}(r_{jk}(X))\right)\right].
\end{equation*}
Consequently, within the \texttt{Arrow} parameterization
\begin{equation*}
r^\theta_{jk}(X)=\nu^\theta_{jk}(X)\sigma(s^\theta_j(X)-s^\theta_k(X)),
\end{equation*}
population training minimizes the edgewise KL projection of the true directed edge marginals onto the skeleton--order model class.
\end{theorem}

\begin{proof}[Proof of Theorem~\ref{thm:composite}]
Condition on $X$ and write $\eta_{jk}=\eta_{jk}(X)$ and $r_{jk}=r_{jk}(X)$. The conditional population risk is
\begin{equation*}
-\sum_{j\neq k}\left[\eta_{jk}\log(r_{jk})+(1-\eta_{jk})\log(1-r_{jk})\right],
\end{equation*}
because $\eta_{jk}=\mathbb{E}[G^\star_{jk}\mid X]$.

For each ordered pair $(j,k)$, the function
\begin{equation*}
r\mapsto-\eta_{jk}\log(r)-(1-\eta_{jk})\log(1-r)
\end{equation*}
is the Bernoulli cross-entropy. It is uniquely minimized over $r\in(0,1)$ at $r=\eta_{jk}$. Hence, the conditional risk is uniquely minimized, pointwise in $X$, by
\begin{equation*}
r^\star_{jk}(X)=\eta_{jk}(X).
\end{equation*}

For the excess-risk identity, subtract the conditional risk at $r^\star$. For each ordered pair,
\begin{equation*}
-\eta_{jk}\log(r_{jk})-(1-\eta_{jk})\log(1-r_{jk})+\eta_{jk}\log(\eta_{jk})+(1-\eta_{jk})\log(1-\eta_{jk})
\end{equation*}
equals
\begin{equation*}
\operatorname{KL}\left(\operatorname{Bern}(\eta_{jk})\,\middle\|\,\operatorname{Bern}(r_{jk})\right).
\end{equation*}
Summing over ordered pairs and taking expectation over $X$ gives
\begin{equation*}
\mathcal{R}(r)-\mathcal{R}(r^\star)=\mathbb{E}\left[\sum_{j\neq k}\operatorname{KL}\left(\operatorname{Bern}(\eta_{jk}(X))\,\middle\|\,\operatorname{Bern}(r_{jk}(X))\right)\right].
\end{equation*}
The final statement follows by restricting $r$ to the \texttt{Arrow} model class
\begin{equation*}
r^\theta_{jk}(X)=\nu^\theta_{jk}(X)\sigma(s^\theta_j(X)-s^\theta_k(X)).
\end{equation*}
\end{proof}

Theorem~\ref{thm:composite} shows that the composite likelihood is a proper scoring rule for the posterior directed edge marginals. Combining this population-risk result with Theorem~\ref{thm:edge-marginal-identifiability} gives the following consistency statement under correct specification.

\begin{corollary}[Population consistency under correct specification]
\label{cor:population-consistency}
Suppose the true conditional distribution over DAGs belongs to the \texttt{Arrow} skeleton--order family, so that for almost every $X$ there exist $\nu^\star_{jk}(X)$ and $s^\star_j(X)$ satisfying
\begin{equation*}
\Pr(G^\star_{jk}=1\mid X)=\nu^\star_{jk}(X)\sigma(s^\star_j(X)-s^\star_k(X)).
\end{equation*}
Then any population minimizer of the composite negative log-likelihood recovers $\nu^\star_{jk}(X)$ and the score differences
\begin{equation*}
s^\star_j(X)-s^\star_k(X)
\end{equation*}
for all $j\neq k$, almost surely. Consequently, it recovers the full latent distribution
\begin{equation*}
p^\star(A,\pi\mid X)=p^\star(A\mid X)p^\star(\pi\mid X).
\end{equation*}
\end{corollary}

\begin{proof}[Proof of Corollary~\ref{cor:population-consistency}]
By Theorem~\ref{thm:composite}, any population minimizer of the composite negative log-likelihood recovers the true directed edge marginals
\begin{equation*}
r^\star_{jk}(X)=\Pr(G^\star_{jk}=1\mid X).
\end{equation*}
Under the assumed skeleton--order specification, these marginals satisfy
\begin{equation*}
r^\star_{jk}(X)=\nu^\star_{jk}(X)\sigma(s^\star_j(X)-s^\star_k(X)).
\end{equation*}
Applying Theorem~\ref{thm:edge-marginal-identifiability} pointwise in $X$ recovers $\nu^\star_{jk}(X)$ and all score differences
\begin{equation*}
s^\star_j(X)-s^\star_k(X).
\end{equation*}
These determine the skeleton distribution and the Plackett--Luce order distribution, with the order scores identified up to an irrelevant additive constant. Hence, the full factorized latent distribution is recovered.
\end{proof}

Corollary~\ref{cor:population-consistency} makes explicit the role of model specification. The result does not claim recovery of arbitrary DAG posteriors. Rather, it states that if the conditional graph law is representable by the skeleton--order family, then no information relevant to that factorized law is lost by replacing the exact likelihood with the directed edge composite likelihood.

\subsection{Composite likelihood interpretation}
\label{app:composite-interpretation}

The composite likelihood replaces the exact DAG likelihood with a product of directed edge marginals. In a general structured prediction setting, such a replacement is often viewed as a simplification that discards global structure, since pairwise marginals do not typically determine the full structured distribution. However, the situation is different for the skeleton--order parameterization used by \texttt{Arrow}. For each unordered pair $\{j,k\}$, the model predicts two directed edge probabilities
\begin{equation*}
r_{jk}=\Pr(G_{jk}=1\mid X),\qquad r_{kj}=\Pr(G_{kj}=1\mid X).
\end{equation*}
Under the factorization $r_{jk}=\nu_{jk}\sigma(s_j-s_k)$ these two quantities encode two distinct components of the latent model:
\begin{enumerate}
\item Their sum recovers the skeleton probability
\begin{equation*}
r_{jk}+r_{kj}=\nu_{jk}.
\end{equation*}
\item Their log-ratio recovers the difference in order scores
\begin{equation*}
\log\left(\frac{r_{jk}}{r_{kj}}\right)=s_j-s_k.
\end{equation*}
\end{enumerate}
Thus, the pair $(r_{jk},r_{kj})$ decomposes into a symmetric component (the skeleton strength) and an antisymmetric component (the directional preference). This decomposition implies that the directed edge marginals are sufficient to reconstruct both latent components of the \texttt{Arrow} model:
\begin{enumerate}
\item The product-Bernoulli skeleton distribution $p_\theta(A\mid X)$ via $\nu_{jk}$.
\item The Plackett--Luce order distribution $p_\theta(\pi\mid X)$ via the score differences $s_j-s_k$.
\end{enumerate}
Since the joint latent distribution factorizes as
\begin{equation*}
p_\theta(A,\pi\mid X)=p_\theta(A\mid X)p_\theta(\pi\mid X),
\end{equation*}
the directed edge marginals determine the full factorized latent distribution, up to the additive shift invariance of the order scores.

In general DAG models, pairwise edge marginals do not determine the joint distribution over graphs, nor do they uniquely specify global structure due to acyclicity and higher-order dependencies. The skeleton--order factorization imposes a special structure in which each pair of directed edge probabilities contains both the undirected connectivity information and the directional ordering information. As a result, the composite likelihood does not discard information relevant to the model class, despite operating on pairwise quantities.

This observation has several important implications:
\begin{itemize}
\item The composite likelihood is not merely a heuristic approximation to the exact likelihood. Within the skeleton--order model class, it targets quantities that are sufficient to identify the latent distribution.
\item The two-head design of \texttt{Arrow} aligns with the decomposition of the directed edge marginals: the skeleton head captures the symmetric component $r_{jk}+r_{kj}$, while the order head captures the antisymmetric component $\log(r_{jk}/r_{kj})$.
\item If the true conditional distribution over DAGs belongs to the skeleton--order family and optimization is ideal, minimizing the composite loss recovers the correct skeleton probabilities and order-score differences, even though the exact DAG likelihood is not optimized.
\item Unlike generic composite likelihood methods, which may lose global dependencies, the directed edge composite likelihood preserves all information needed to reconstruct the latent skeleton--order model.
\item The result does not imply recovery of arbitrary DAG posteriors. It holds under correct specification of the skeleton--order model class. Outside this class, the composite likelihood learns the closest directed edge marginals representable within the model.
\end{itemize}

Taken together, these points show that the composite likelihood is well aligned with the structure of the \texttt{Arrow} model: it reduces computational complexity while retaining the information necessary to identify the latent factorization that generates DAGs.

\subsection{Misspecification and global consistency}
\label{app:misspecification}

Having established why the directed edge composite likelihood is justified within the skeleton--order model class, we now characterize that model class directly in terms of directed edge marginals and explain what happens under misspecification.

For each ordered pair $(j,k)$ with $j\neq k$, define the directed edge marginal
\begin{equation*}
r_{jk}:=\Pr_\theta(G_{jk}=1\mid X),
\end{equation*}
and the corresponding directional log-ratio
\begin{equation*}
\rho_{jk}:=\log\left(\frac{r_{jk}}{r_{kj}}\right).
\end{equation*}
Under the \texttt{Arrow} factorization,
\begin{equation*}
r_{jk}=\nu_{jk}\sigma(s_j-s_k),
\end{equation*}
where $\nu_{jk}=\nu_{kj}$ and $\sigma(\cdot)$ is the logistic sigmoid. The symmetry of the skeleton probabilities implies that the directional log-ratios depend only on score differences:
\begin{equation*}
\rho_{jk}=\log\left(\frac{\nu_{jk}\sigma(s_j-s_k)}{\nu_{jk}\sigma(s_k-s_j)}\right)=s_j-s_k.
\end{equation*}
Thus, although \texttt{Arrow} uses a distribution over orderings, its pairwise ordering preferences are governed by a single global score vector $s=(s_1,\dots,s_p)$. This structure imposes consistency constraints across triples of variables. For any distinct nodes $i,j,k$,
\begin{equation*}
\rho_{ij}+\rho_{jk}+\rho_{ki}=(s_i-s_j)+(s_j-s_k)+(s_k-s_i)=0.
\end{equation*}
Intuitively, if node $i$ tends to precede node $j$, and node $j$ tends to precede node $k$, then the preference between $i$ and $k$ must be compatible with the same global score system. \texttt{Arrow} can represent uncertainty over orderings, but not arbitrary cyclic inconsistencies in pairwise directional preferences.

The following proposition shows that this consistency condition completely characterizes the \texttt{Arrow} model class at the level of directed edge marginals.

\begin{proposition}[Global consistency characterization]
\label{prop:curl-free}
Fix $X$ and let $\{r_{jk}\}_{j\neq k}$ be directed edge marginals satisfying
\begin{equation*}
r_{jk}>0,\qquad r_{jk}+r_{kj}<1,\qquad j\neq k.
\end{equation*}
Define
\begin{equation*}
\rho_{jk}:=\log\left(\frac{r_{jk}}{r_{kj}}\right).
\end{equation*}
Then the following statements are equivalent:
\begin{enumerate}
\item There exist skeleton probabilities $\nu_{jk}=\nu_{kj}\in(0,1)$ and node scores $s_1,\dots,s_p$ such that
\begin{equation*}
r_{jk}=\nu_{jk}\sigma(s_j-s_k),\qquad j\neq k.
\end{equation*}
\item For every triple of distinct nodes $(i,j,k)$,
\begin{equation*}
\rho_{ij}+\rho_{jk}+\rho_{ki}=0.
\end{equation*}
\end{enumerate}
\end{proposition}

\begin{proof}
Suppose first that
\begin{equation*}
r_{jk}=\nu_{jk}\sigma(s_j-s_k).
\end{equation*}
Then
\begin{equation*}
\rho_{jk}=\log\left(\frac{r_{jk}}{r_{kj}}\right)=\log\left(\frac{\nu_{jk}\sigma(s_j-s_k)}{\nu_{jk}\sigma(s_k-s_j)}\right)=s_j-s_k,
\end{equation*}
where we used $\sigma(a)/\sigma(-a)=\exp(a)$. Therefore, for any triple $(i,j,k)$,
\begin{equation*}
\rho_{ij}+\rho_{jk}+\rho_{ki}=(s_i-s_j)+(s_j-s_k)+(s_k-s_i)=0.
\end{equation*}

Conversely, suppose for every triple $(i,j,k)$ the consistency condition holds:
\begin{equation*}
\rho_{ij}+\rho_{jk}+\rho_{ki}=0.
\end{equation*}
Fix a reference node, say node $1$, and define
\begin{equation*}
s_j:=\rho_{j1},\qquad j=2,\dots,p.
\end{equation*}
Then for any pair $(j,k)$, applying the consistency condition to the triple $(j,k,1)$ gives
\begin{equation*}
\rho_{jk}+\rho_{k1}+\rho_{1j}=0.
\end{equation*}
Since $\rho_{1j}=-\rho_{j1}$, this equality implies
\begin{equation*}
\rho_{jk}=\rho_{j1}-\rho_{k1}=s_j-s_k.
\end{equation*}
Now define
\begin{equation*}
\nu_{jk}:=r_{jk}+r_{kj}.
\end{equation*}
By assumption, $\nu_{jk}\in(0,1)$ and $\nu_{jk}=\nu_{kj}$. Since
\begin{equation*}
\frac{r_{jk}}{r_{kj}}=e^{\rho_{jk}}=e^{s_j-s_k},
\end{equation*}
we have
\begin{equation*}
\frac{r_{jk}}{\nu_{jk}}=\frac{r_{jk}}{r_{jk}+r_{kj}}=\frac{\exp(s_j-s_k)}{1+\exp(s_j-s_k)}=\sigma(s_j-s_k).
\end{equation*}
Hence
\begin{equation*}
r_{jk}=\nu_{jk}\sigma(s_j-s_k),
\end{equation*}
which is the \texttt{Arrow} factorization.
\end{proof}

Proposition~\ref{prop:curl-free} gives a precise characterization of model misspecification. A collection of directed edge marginals is representable by \texttt{Arrow} if and only if its pairwise directional preferences are globally consistent with a single set of order scores. This condition is stronger than requiring individual sampled graphs to be acyclic: a distribution over valid DAGs can still induce pairwise directional marginals that are cyclic in aggregate and therefore cannot be represented by a single score vector.

When the true directed edge marginals violate the consistency condition in Proposition~\ref{prop:curl-free}, they lie outside the skeleton--order model class. In that case, the composite likelihood still has a natural interpretation, as it learns the best representable approximation to the true directed edge marginals under the edgewise KL objective of Theorem~\ref{thm:composite}. Equivalently, \texttt{Arrow} replaces arbitrary pairwise directional preferences with the closest globally coherent preferences expressible through symmetric skeleton probabilities and a single score vector. Thus, misspecification does not make the objective meaningless, but rather specifies which part of the target distribution \texttt{Arrow} can approximate.

The consistency condition in Proposition~\ref{prop:curl-free} can be viewed as a discrete integrability constraint on pairwise preferences: the log-ratio field $\rho_{jk}$ is representable as differences of a potential $s$ if and only if all cycle sums vanish. This condition is the finite-graph analogue of a curl-free (conservative) field in vector calculus. Related perspectives appear in the theory of paired comparisons and ranking models, where Bradley--Terry--Luce/Plackett--Luce models induce exactly such difference-based structures on log-odds \citep{Bradley1952,Luce1959,Plackett1975}. Connections to graph-based Hodge decompositions, which separate gradient (potential) components from cyclic inconsistencies on graphs, provide a complementary view \citep{Jiang2011}. Exploring these geometric connections for causal graph posteriors is an interesting direction for future work.

\section{Architecture details}
\label{app:model}

\subsection{Configuration}

\texttt{Arrow} uses variable embedding dimension $d=512$. The transformer modules all share the same configuration, each consisting of 3 transformer blocks with residual connections, 8 attention heads, feedforward dimension $4d$, and layer normalization \citep{Ba2016}. For the variable transformer, we use $m=16$ summary tokens per variable. For the skeleton head MLP, we use one hidden layer with dimension $2d$. GELU activations \citep{Hendrycks2016} are used throughout the network. With this specific configuration of the network, \texttt{Arrow} has approximately $37$M trainable parameters.

\subsection{Complexity}

With the model configuration fixed, \texttt{Arrow} scales linearly in the number of observations and quadratically in the number of variables. The observation transformer attends across variables within each observation, the variable transformer uses learned summary tokens to aggregate observations for each variable, and the context transformer attends across the resulting embeddings. The skeleton head evaluates all pairs of variables, adding another quadratic term in the number of variables. The computational complexity is thus $O(np^2+p^2)$ and the memory complexity is $O(np+np^2+p^2)$.

\section{Pretraining details}
\label{app:pretraining}

\subsection{Compute environment}

Training is performed with automatic mixed precision in bfloat16 and distributed data parallelism across 32 NVIDIA H100 GPUs. Memory usage varies with both $n$ and $p$, so we use shape-aware microbatching. For each sampled shape, we use the largest microbatch size that fits in memory, accumulate gradients over microbatches of that shape, and rescale them so that each optimizer iteration matches the target batch size. Feasible microbatch sizes are computed before training begins.

\subsection{Task streaming}

\texttt{Arrow} is trained over a synthetic distribution of causal discovery tasks. The relevant objective is therefore the population risk under the task generator, rather than the empirical risk over a fixed collection of simulated datasets. Each task is a dataset--graph pair $(X,G^\star)$, and training estimates the expected composite loss under the task distribution. In this setting, repeated exposure to the same tasks primarily reduces within-task gradient variance, whereas sampling new tasks improves coverage of the training distribution over graph families, functional forms, noise models, and dataset shapes. We therefore generate fresh tasks throughout training and expose the model to each task only once.

Our streaming setup is aligned with tabular foundation models such as \texttt{TabPFN} and \texttt{TabICL}, which train on synthetic datasets from priors over data-generating processes to learn reusable inference maps \citep{Muller2022,Hollmann2023,Qu2025}. A similar perspective appears in simulation-based inference, where neural estimators are trained on simulated pairs from a generative model, and in stochastic optimization, where objectives defined as expectations are approximated with fresh samples \citep{Cranmer2020,Bottou2018}. More broadly, language-model scaling analyses show that, under fixed compute, allocating compute to additional training data rather than further optimization on the same data is often beneficial \citep{Kaplan2020,Hoffmann2022}.

\subsection{Training trajectory}

In Figure~\ref{fig:training_trajectory}, we present the trajectories of the negative log-likelihood over the course of the training iterations. As a validation set, we sample $10{,}000$ tasks from the training data-generating distribution (see Section~\ref{sec:data-generation}) and fix this set prior to the start of training. The training data are streamed during optimization, so the raw training negative log-likelihood can fluctuate and is shown as a faint blue line. A 25-iteration rolling mean is overlaid as a dark blue line, showing relatively stable progress.

\begin{figure}[ht]
\centering
\includegraphics[width=0.5\linewidth]{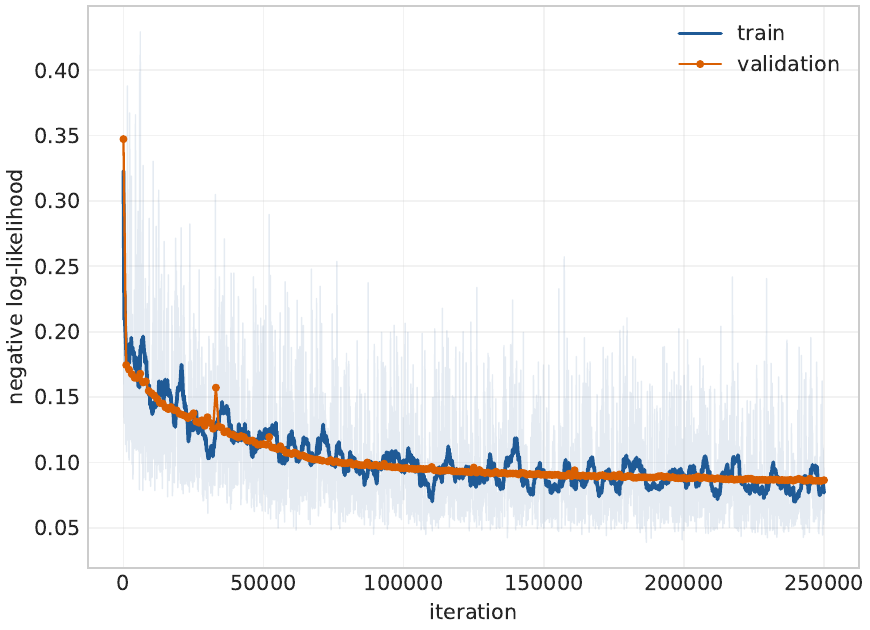}
\caption{Negative log-likelihood trajectories on the training and validation data.}
\label{fig:training_trajectory}
\end{figure}

\section{Experimental setup and additional results}

\subsection{Setup}
\label{app:setup}

Inference for the experiments is run on a machine with an AMD EPYC 9554 processor, 376GB RAM, and 4 NVIDIA RTX 5090 GPUs. Each method is allocated a single GPU or CPU core, with the experiments run in parallel across GPUs and CPU cores. Allocating multiple CPU cores to the CPU-based methods did not materially affect runtime. The experiments take roughly 10 hours to run.

We use public implementations of all baselines. The repository, version, and license are listed below.
\begin{itemize}
\item \texttt{AVICI}: \href{https://github.com/larslorch/avici}{avici}, v1.0.7, MIT License
\item \texttt{BCNP}: \href{https://github.com/Anish144/CausalStructureNeuralProcess}{CausalStructureNeuralProcess}, v0.0.1, MIT License
\item \texttt{SEA}: \href{https://github.com/rmwu/sea-reproduce}{sea-reproduce}, no version or license reported
\item \texttt{DAGMA}: \href{https://github.com/kevinsbello/dagma}{dagma}, v1.1.1, Apache-2.0 License
\item \texttt{SDCD}: \href{https://github.com/azizilab/sdcd}{sdcd}, v0.1.4, MIT License
\end{itemize}
Some benchmark datasets are constructed using external generators. Their details are listed below.
\begin{itemize}
\item Bayesian Network Repository data; \citet{Sachs2005} data: \href{https://github.com/pgmpy/pgmpy}{pgmpy}, v1.1.0, MIT License
\item \citet{Herman2025} data: \href{https://github.com/rebeccaherman1/UUMCdata}{UUMCdata}, v0.0, GPL-3.0 License
\end{itemize}

\subsection{Additional in-distribution experiments}
\label{app:in_distribution}

Figure~\ref{fig:synthetic_additional} extends the main experiments on in-distribution data by systematically isolating the remaining generator components, namely the graph family, functional form, and noise model. \texttt{Arrow} remains consistently stable and competitive across these variations, suggesting that its performance is not an artifact of any particular generator setting. Instead, the results indicate that broad pretraining diversity translates into robust zero-shot graph prediction across the full synthetic task distribution.

\begin{figure}[ht]
\centering
\includegraphics[width=\linewidth]{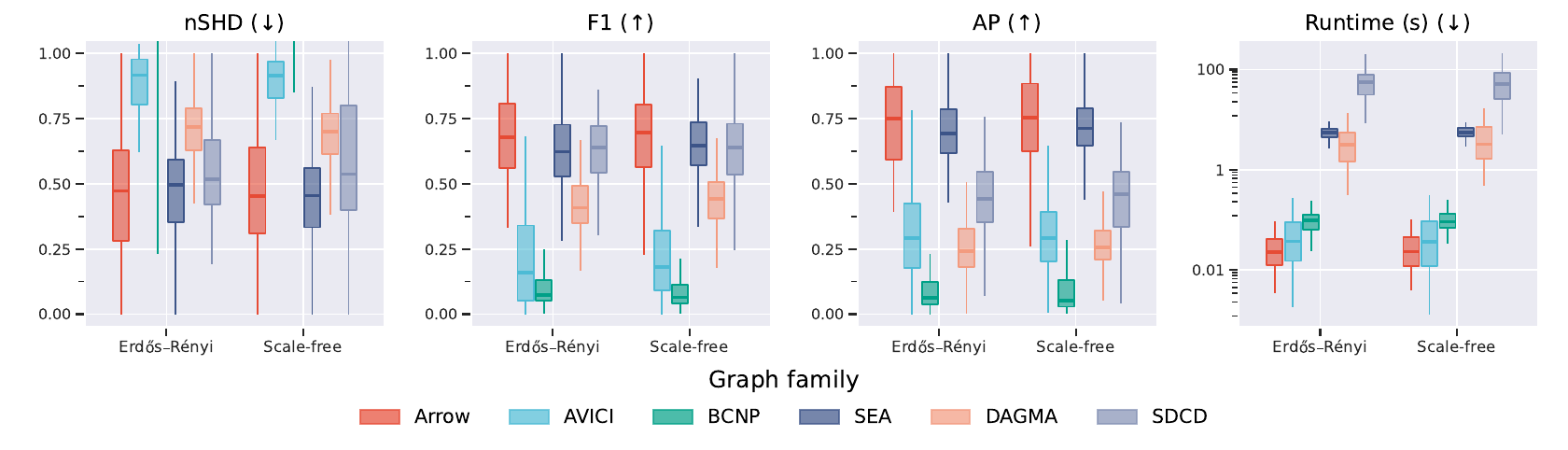}
\includegraphics[width=\linewidth]{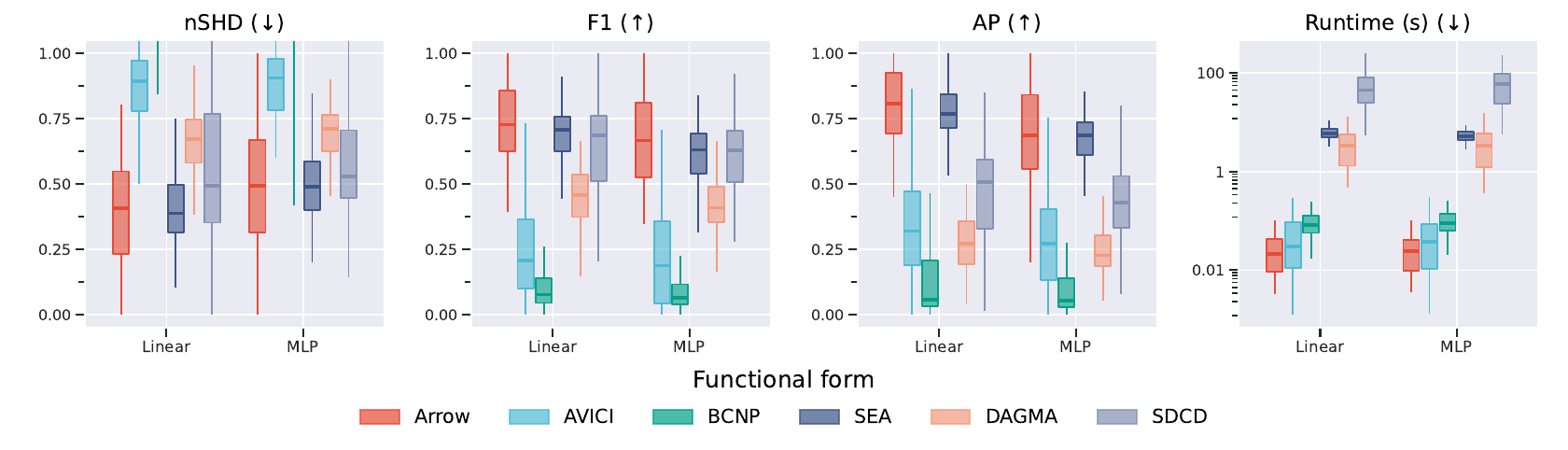}
\includegraphics[width=\linewidth]{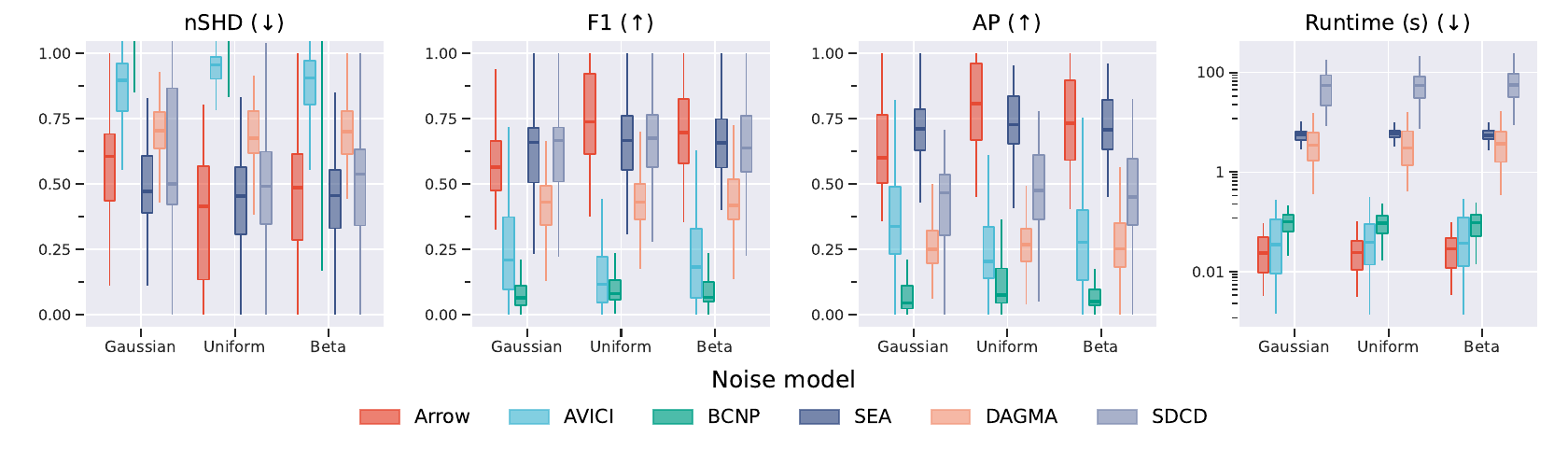}
\caption{Performance across graph families, functional forms, and noise models on in-distribution synthetic datasets. The averages (solid points) and standard errors (error bars) summarize results over 100 datasets. The nSHD axis is truncated at 2 to preserve resolution among stronger baselines.}
\label{fig:synthetic_additional}
\end{figure}

\subsection{Additional out-of-distribution experiments}
\label{app:out_of_distribution}

As an additional robustness check, we evaluate on synthetic datasets produced by the data-generating procedures of \citet{Herman2025}. This benchmark uses linear Gaussian data generated over Erdős--Rényi DAGs, but differs in lower-level choices such as coefficient scaling, noise variance, data standardization, and signal-to-noise structure. For each dataset, we sample the number of observations, number of variables, and number of edges using the same ranges as in Section~\ref{sec:data-generation}. Table~\ref{tab:herman} shows that \texttt{Arrow} performs well under these alternative data-generating specifications, often achieving the best nSHD, F1, and AP while retaining sub-second inference. These results demonstrate that its gains are not tied to the specific details of our generator, but persist across independent benchmarks.

\begin{table}[ht]
\centering
\caption{Performance on synthetic datasets from \citet{Herman2025}. The averages and standard errors summarize results over 100 datasets. Bold indicates the best result per metric, including ties whose standard error intervals overlap with the best entry.}
\label{tab:herman}
\scriptsize
\begingroup
\setlength{\tabcolsep}{4pt}
\begin{minipage}[t]{0.49\textwidth}
\centering
\vspace{0.3em}UUMC \par\vspace{0.3em}
\begin{tabular}{lp{0.47in}p{0.47in}p{0.47in}p{0.47in}}
\toprule
 & nSHD (↓) & F1 (↑) & AP (↑) & Time (↓) \\
\midrule
\texttt{Arrow} & \textbf{0.66 (0.02)} & \textbf{0.49 (0.02)} & \textbf{0.50 (0.01)} & \textbf{0.03 (0.00)} \\
\texttt{AVICI} & 0.84 (0.03) & 0.30 (0.02) & 0.39 (0.02) & 0.07 (0.01) \\
\texttt{BCNP}  & 6.58 (0.73) & 0.08 (0.01) & 0.10 (0.01) & 0.10 (0.01) \\
\texttt{SEA}   & 0.77 (0.03) & 0.45 (0.02) & 0.42 (0.02) & 5.37 (0.17) \\
\texttt{DAGMA} & 0.87 (0.02) & 0.31 (0.02) & 0.18 (0.02) & 5.10 (0.40) \\
\texttt{SDCD}  & 1.05 (0.15) & \textbf{0.47 (0.02)} & 0.30 (0.02) & 63.76 (4.33) \\
\bottomrule
\end{tabular}
\end{minipage}
\hfill
\begin{minipage}[t]{0.49\textwidth}
\centering
\vspace{0.3em}unit-variance-noise \par\vspace{0.3em}
\begin{tabular}{lp{0.47in}p{0.47in}p{0.47in}p{0.47in}}
\toprule
 & nSHD (↓) & F1 (↑) & AP (↑) & Time (↓) \\
\midrule
\texttt{Arrow} & 0.74 (0.04) & \textbf{0.46 (0.03)} & 0.44 (0.03) & \textbf{0.03 (0.00)} \\
\texttt{AVICI} & \textbf{0.67 (0.02)} & \textbf{0.47 (0.02)} & \textbf{0.59 (0.02)} & 0.07 (0.01) \\
\texttt{BCNP}  & 7.78 (1.07) & 0.08 (0.01) & 0.09 (0.01) & 0.10 (0.01) \\
\texttt{SEA}   & 1.87 (0.13) & 0.26 (0.02) & 0.21 (0.02) & 6.87 (0.25) \\
\texttt{DAGMA} & 0.96 (0.03) & 0.31 (0.02) & 0.19 (0.02) & 7.27 (0.60) \\
\texttt{SDCD}  & 1.38 (0.14) & 0.33 (0.02) & 0.19 (0.02) & 75.10 (4.90) \\
\bottomrule
\end{tabular}
\end{minipage}

\medskip

\begin{minipage}[t]{0.49\textwidth}
\centering
\vspace{0.3em}iSCM \par\vspace{0.3em}
\begin{tabular}{lp{0.47in}p{0.47in}p{0.47in}p{0.47in}}
\toprule
 & nSHD (↓) & F1 (↑) & AP (↑) & Time (↓) \\
\midrule
\texttt{Arrow} & \textbf{0.53 (0.02)} & \textbf{0.61 (0.02)} & \textbf{0.67 (0.02)} & \textbf{0.03 (0.00)} \\
\texttt{AVICI} & 0.72 (0.04) & 0.43 (0.02) & 0.56 (0.02) & 0.07 (0.01) \\
\texttt{BCNP}  & 7.13 (0.84) & 0.09 (0.01) & 0.11 (0.02) & 0.10 (0.01) \\
\texttt{SEA}   & \textbf{0.51 (0.02)} & \textbf{0.62 (0.02)} & 0.62 (0.02) & 6.19 (0.20) \\
\texttt{DAGMA} & 0.76 (0.02) & 0.45 (0.02) & 0.28 (0.02) & 5.48 (0.43) \\
\texttt{SDCD}  & 0.93 (0.13) & 0.58 (0.02) & 0.40 (0.02) & 55.36 (3.36) \\
\bottomrule
\end{tabular}
\end{minipage}
\hfill
\begin{minipage}[t]{0.49\textwidth}
\centering
\vspace{0.3em}IPA \par\vspace{0.3em}
\begin{tabular}{lp{0.47in}p{0.47in}p{0.47in}p{0.47in}}
\toprule
 & nSHD (↓) & F1 (↑) & AP (↑) & Time (↓) \\
\midrule
\texttt{Arrow} & \textbf{0.49 (0.03)} & \textbf{0.64 (0.02)} & \textbf{0.69 (0.02)} & \textbf{0.03 (0.00)} \\
\texttt{AVICI} & 0.72 (0.02) & 0.40 (0.02) & 0.55 (0.02) & 0.07 (0.01) \\
\texttt{BCNP}  & 7.26 (0.99) & 0.08 (0.01) & 0.10 (0.01) & 0.10 (0.01) \\
\texttt{SEA}   & \textbf{0.46 (0.02)} & \textbf{0.65 (0.02)} & \textbf{0.68 (0.02)} & 6.01 (0.19) \\
\texttt{DAGMA} & 0.72 (0.02) & 0.46 (0.02) & 0.29 (0.02) & 5.13 (0.40) \\
\texttt{SDCD}  & 0.80 (0.09) & 0.61 (0.02) & 0.44 (0.02) & 53.08 (3.25) \\
\bottomrule
\end{tabular}
\end{minipage}

\medskip

\begin{minipage}[t]{0.49\textwidth}
\centering
\vspace{0.3em}50-50 \par\vspace{0.3em}
\begin{tabular}{lp{0.47in}p{0.47in}p{0.47in}p{0.47in}}
\toprule
 & nSHD (↓) & F1 (↑) & AP (↑) & Time (↓) \\
\midrule
\texttt{Arrow} & 0.45 (0.02) & 0.66 (0.02) & 0.72 (0.02) & \textbf{0.03 (0.00)} \\
\texttt{AVICI} & 0.79 (0.02) & 0.31 (0.02) & 0.44 (0.02) & 0.07 (0.01) \\
\texttt{BCNP}  & 6.52 (0.81) & 0.08 (0.01) & 0.10 (0.02) & 0.10 (0.01) \\
\texttt{SEA}   & \textbf{0.31 (0.02)} & \textbf{0.75 (0.02)} & \textbf{0.82 (0.02)} & 5.56 (0.16) \\
\texttt{DAGMA} & 0.56 (0.02) & 0.53 (0.02) & 0.37 (0.02) & 4.14 (0.30) \\
\texttt{SDCD}  & 0.59 (0.10) & 0.71 (0.02) & 0.57 (0.02) & 51.76 (3.21) \\
\bottomrule
\end{tabular}
\end{minipage}
\hfill
\begin{minipage}[t]{0.49\textwidth}
\centering
\vspace{0.3em}DaO \par\vspace{0.3em}
\begin{tabular}{lp{0.47in}p{0.47in}p{0.47in}p{0.47in}}
\toprule
 & nSHD (↓) & F1 (↑) & AP (↑) & Time (↓) \\
\midrule
\texttt{Arrow} & \textbf{0.82 (0.02)} & 0.29 (0.02) & \textbf{0.34 (0.02)} & \textbf{0.03 (0.00)} \\
\texttt{AVICI} & 0.91 (0.02) & 0.19 (0.02) & 0.25 (0.02) & 0.07 (0.01) \\
\texttt{BCNP}  & 6.89 (1.13) & 0.08 (0.01) & 0.10 (0.02) & 0.11 (0.01) \\
\texttt{SEA}   & \textbf{0.86 (0.03)} & 0.30 (0.02) & \textbf{0.32 (0.02)} & 4.35 (0.13) \\
\texttt{DAGMA} & 0.92 (0.02) & 0.19 (0.02) & 0.14 (0.02) & 3.62 (0.25) \\
\texttt{SDCD}  & 1.01 (0.10) & \textbf{0.41 (0.02)} & 0.25 (0.02) & 59.89 (3.94) \\
\bottomrule
\end{tabular}
\end{minipage}
\endgroup
\end{table}

Finally, we evaluate on the flow cytometry dataset from \citet{Sachs2005}, which is commonly used as a real data benchmark for causal discovery. The dataset contains $p=11$ phosphoproteins and phospholipids measured from primary human immune cells, with predictions evaluated against the corresponding expert consensus graph with $s=20$ edges. Table~\ref{tab:sachs} demonstrates that \texttt{Arrow} transfers to this real biological setting, attaining the best nSHD, tying for the best F1 and fastest inference time, and achieving competitive AP. Together, these results show that \texttt{Arrow}'s zero-shot advantage is not confined to synthetic or semi-synthetic generators, but can carry over to real biological data.

\begin{table}[ht]
\centering
\caption{Performance on the real flow cytometry dataset from \citet{Sachs2005}. The averages and standard errors summarize results over 100 subsamples. Bold indicates the best result per metric, including ties whose standard error intervals overlap with the best entry.}
\label{tab:sachs}
\scriptsize
\begingroup
\setlength{\tabcolsep}{3.6pt}
\vspace{1em}SACHS ($p=11$, $s=20$, $n=100$) \par\vspace{0.3em}
\begin{tabular}{lp{0.47in}p{0.47in}p{0.47in}p{0.47in}}
\toprule
 & nSHD (↓) & F1 (↑) & AP (↑) & Time (↓) \\
\midrule
\texttt{Arrow} & \textbf{1.06 (0.01)} & \textbf{0.27 (0.01)} & 0.32 (0.01) & \textbf{0.01 (0.00)} \\
\texttt{AVICI} & 1.10 (0.02) & \textbf{0.27 (0.01)} & 0.30 (0.01) & \textbf{0.01 (0.00)} \\
\texttt{BCNP}  & 1.09 (0.01) & 0.19 (0.00) & \textbf{0.35 (0.01)} & 0.03 (0.00) \\
\texttt{SEA}   & 1.14 (0.01) & 0.16 (0.01) & 0.22 (0.00) & 7.55 (0.03) \\
\texttt{DAGMA} & 1.09 (0.01) & 0.13 (0.01) & 0.18 (0.00) & 0.87 (0.01) \\
\texttt{SDCD}  & 1.59 (0.02) & \textbf{0.25 (0.01)} & 0.20 (0.00) & 6.99 (0.18) \\
\bottomrule
\end{tabular}
\endgroup
\end{table}

\section{Limitations and societal impacts}
\label{app:limitations}

\texttt{Arrow} is designed for causal discovery from observational tabular data and therefore inherits the usual identifiability limitations of this setting. Its predictions are most reliable when standard observational causal discovery assumptions and the synthetic training distribution are approximately satisfied, and may degrade under latent confounding, selection bias, cyclicity, mixed data types, or substantial distribution shift. The current model is trained on datasets with up to $2000$ observations and $100$ variables, and its quadratic scaling in the number of variables may limit use on larger datasets. Although the skeleton--order factorization guarantees acyclic outputs, it does not guarantee causal correctness. These limitations motivate extensions to broader causal settings and larger datasets.

Causal discovery tools support scientific inquiry by helping researchers generate hypotheses about causal structure in numerous application domains. By training \texttt{Arrow} as a reusable pretrained model for causal discovery, such analyses may become faster and more accessible. However, as with any method that infers causal graphs from observational data, its outputs should not be interpreted as definitive causal relationships without domain knowledge and assumptions such as causal sufficiency and faithfulness. Misinterpretation of learned graphs could lead to incorrect downstream conclusions, with particularly consequential impacts in high-stakes settings such as health, policy, or resource allocation. Accordingly, \texttt{Arrow} is best used as a tool in conjunction with subject-matter expertise.

\end{document}